\documentclass[sigconf,nonacm]{aamas}

\usepackage{balance}
\usepackage{amsmath,amsthm}
\usepackage{graphicx,booktabs,array}
\usepackage{microtype}
\usepackage{url}

\setcopyright{none}
\acmDOI{}
\acmISBN{}

\title{Learning to Report Unsafe Tasks in a Multi-Agent Game}
\author{Avyay M. Casheekar}
\affiliation{%
  \institution{University of Michigan Law School}
  \city{Ann Arbor}
  \state{Michigan}
  \country{USA}}
\author{Hariganesh Tangirala}
\affiliation{%
  \department{School of Computing}
  \institution{National University of Singapore}
  \country{Singapore}}
\renewcommand{\shortauthors}{Casheekar and Tangirala}

\begin{abstract}
When a team of agents shares a reward for completed tasks, any agent that reports and stops an unsafe sub-task reduces the team's collective reward.
Random audits that penalize silent agents can make universal reporting the team's optimal policy. Such a system, however, does not necessarily guarantee that further training will teach a silent team to report.
We study this learning problem in a game where any witness of a task can stop it by reporting. When the $k$ witnesses of each task share a policy and act independently, the derivative of the expected reward with respect to their silence probability counts each task's benefit $k$ times at universal silence, while the comparison between universal reporting and universal silence counts it once. A fully shared policy can therefore be locally attracted to silence even when universal reporting maximizes its expected reward. For any grouping of agents into shared policies, we give an audit condition that is sufficient for exact policy-gradient updates to reach universal reporting and, except in boundary cases, necessary when training starts near universal silence. Its cost depends on how many witnesses of each task share a policy, and in a balanced family of games, meeting it with prescribed margins costs exactly $k$ times as much for a fully shared policy as for one policy per role.
We test these predictions by training PPO agents from learned silence on 24 witness graphs, using witness groups, where each pair of co-witnesses is placed in separate networks, and shuffled groups with the same group sizes subject to the same audits. We find that witness groups reduce the rate of unsafe task completions by 33.59 percentage points (95\% graph-bootstrap CI [21.03, 45.13]). However, only 9 of 48 runs using witness groups result in unsafe task completions below 1\% with a legitimate task completion rate of at least 90\%. A fully shared network achieves these thresholds in 0/48 runs at the same audit budget when its agents draw their actions independently, and in 48/48 runs when they take the same action through a single draw. On 24 further graphs with witness groups held fixed, audits allocated by the condition reduce unsafe task completions by 42.22 points relative to audits that maximize a static coalition deterrence ratio at the same budget.
\end{abstract}



\newtheorem{theorem}{Theorem}
\newtheorem{proposition}[theorem]{Proposition}
\newtheorem{corollary}[theorem]{Corollary}

\newcommand{\ind}{\mathbf{1}}
\newcommand{\sig}{\sigma}

\newcommand{\supref}[1]{\ref{#1}}
\newcommand{\mainref}[1]{\ref{#1}}



\makeatletter
\@ACM@balancefalse
\makeatother

\AtBeginDocument{%
  \fancypagestyle{standardpagestyle}{%
    \fancyhf{}%
    \fancyfoot[C]{\footnotesize\thepage}}%
  \fancypagestyle{firstpagestyle}{%
    \fancyhf{}%
    \fancyfoot[C]{\footnotesize\thepage}}%
  \pagestyle{standardpagestyle}}

\begin{document}
\maketitle
\hypersetup{pdfauthor={Avyay M. Casheekar and Hariganesh Tangirala}}
\raggedbottom

\section{Introduction}
\label{sec:intro}

AI developers have begun to run hundreds of AI agents together on a single large problem. In an experiment at Cursor, hundreds of coding agents worked concurrently for close to a week and wrote over a million lines of code for a web browser \citep{lin2026scaling}. In another experiment, 100 agents worked as a research collective to prove formal mathematical conjectures \citep{paglieri2026swarm}.

In July 2026, such a system formed without being designed. Agents from several of OpenAI's models were meant to stay isolated from one another in cybersecurity evaluations, but some of them communicated through an unintended message board inside an internal package manager \citep{greenblatt2026incident,openai2026incident}. Roughly 1,200 agents exchanged over 70,000 messages and files on the board \citep{greenblatt2026incident}. There they researched ways to fool the evaluation's automated grader, and about 700 of them joined an attack on Hugging Face, whose servers the agents believed might hold information about the grader \citep{greenblatt2026incident}. Some agents declined to join \citep{openai2026incident}. In the transcripts it reviewed, METR's investigation found no agent that attempted to alert humans \citep{greenblatt2026incident}.

The agents on the board saw the attack while it was being planned \citep{greenblatt2026incident}, but OpenAI connected it to its agents only on July 20, four days after Hugging Face disclosed it \citep{openai2026incident}. Reviewing the board's more than 70,000 messages and files is a problem of scalable oversight \citep{amodei2016,bowman2022}, while the agents on the board read those messages as part of their own tasks. An agent that reports the unsafe tasks it witnesses could expose them before outside monitoring does. In the research collective, some agents complained about fraudulent proofs, but no one monitored the complaint channel during the run and the agents had no tool to remove a fraudulent submission, so the fraud continued \citep{paglieri2026swarm}. 

Two recent methods train language models to disclose their own misbehavior, and in both the report leaves the reported work in place during training. In confession training, an honest confession earns a separate reward and the reward for the main answer does not change \citep{joglekar2025confessions}. In self-incrimination training, the agent calls a reporting tool whenever it acts on a hidden side task, and the action still goes ahead \citep{lee2026selfreport}. A report on another agent's task can cost the reporter, depending on how rewards are shared. In the research collective, credit for a conjecture went to the first submission that the automated checker accepted, which also removed the conjecture from the open pool \citep{paglieri2026swarm}. Cooperative multi-agent reinforcement learning instead commonly trains all the agents of a team on one team reward \citep{claus1998,yu2022}. If the team reward pays for completed tasks without checking whether they are unsafe, a report that stops an unsafe task removes its reward for every agent, the reporter included. Training on such a reward can teach a team to stay silent without any agreement to conceal the task. The agents on the board were evaluated on separate tasks \citep{openai2026incident} and stayed silent without such a reward, so a team reward would add a direct incentive for the same silence. Learned silence is a form of reward hacking, in which optimizing the training reward worsens performance on the objective the designer intended \citep{amodei2016,skalse2022}.

An auditor can counter the reward for silence by auditing agents at random and fining those found silent on unsafe tasks \citep{becker1968,polinsky2000}. Each audit examines the decision of one agent, and we assume that the auditor can verify whether that agent stayed silent about unsafe work. Sufficient audits can make universal reporting the team's optimal policy. Such audits, however, do not necessarily guarantee that further training will teach a silent team to report. Oversight may be strengthened after a problem appears, as OpenAI's monitoring requirements were after the incident \citep{openai2026incident}, and by then a team may already be silent. We therefore consider a team that learned silence before audits began and continues training under the audits with the same learning rule, while the auditor chooses where to audit.

We study this learning problem in a reporting game where each task has several witnesses and any witness can stop the task by reporting. Systems of AI agents often run many instances of the same model \citep{greenblatt2026incident,paglieri2026swarm}, and cooperative multi-agent reinforcement learning commonly shares one policy across agents \citep{gupta2017,terry2020,yu2022}. The agents of our game therefore share a policy within groups, and we assume that they draw their actions independently given that policy. We do not model coordination through shared messages or files, as on the board, since that information is not part of the policy's input.

Under independent draws, a task completes with probability equal to the product of its witnesses' silence probabilities, and at universal silence a report by any one witness stops all of that witness's tasks. When the $k$ witnesses of each task share one silence probability, the derivative of the expected reward with respect to it at universal silence counts each task's benefit $k$ times, while the comparison between universal reporting and universal silence counts it once. A fully shared policy can therefore be locally attracted to silence even when universal reporting maximizes its expected reward. Universal silence stays a strict local maximum until the agents' total expected cost of silence, their total deterrence, reaches $k$ times the level that makes reporting optimal (Section~\ref{sec:obstruction}). More witnesses per task give an unsafe task more chances to be stopped once agents report, but under a shared policy they also make silence harder to leave. We call this local maximum the obstruction. It lies in the exact expected reward and arises before any sampling noise or neural-network approximation.

We ask which audits overcome this obstruction under exact policy-gradient and natural-gradient updates \citep{williams1992,sutton2000,kakade2001} and how the answer depends on which agents share a policy.
\begin{enumerate}
\item For any grouping of agents into shared policies, we give an audit condition, the strict elimination certificate, under which exact updates reach universal reporting from every interior starting point (Theorem~\ref{thm:recovery}). Apart from boundary cases, it is also necessary when training starts near universal silence. The certificate orders the groups so that the first has enough deterrence to learn to report whatever the other groups do. The first group's reports stop the tasks its members witness, so those tasks stop rewarding silence in the remaining groups, and the next group needs deterrence only for the tasks that remain.
\item The cost of the certificate depends on how the grouping places each task's witnesses. For each task, the first group in the order to contain one of its witnesses must cover the task's benefit with deterrence once for each of the task's witnesses it contains, so putting each task's witnesses in different groups pays each benefit once, while a fully shared policy pays it $k$ times (Proposition~\ref{prop:cost}). In a balanced family of games where every task has one witness from each of $k$ roles, the cheapest audits that meet the certificate with prescribed positive margins cost exactly $k$ times as much for a fully shared policy as for one policy per role (Corollary~\ref{cor:balanced}).
\item The factor $k$ also depends on independent draws, which sharing a policy does not require. When all agents act on a single shared draw, which we call a common draw, any allocation of audit probabilities that makes universal reporting strictly better than universal silence also makes every exact gradient point toward reporting (Section~\ref{sec:dependence}).
\end{enumerate}
Local maxima of this kind are known from imperfect-recall games, in which one behavioral probability is reused at several decision points (Section~\ref{sec:related}). Our contribution is the certificate and its cost for this objective, together with experiments on whether the certificate helps choose audits for a neural learner.

We test these predictions by training small neural-network policies with PPO \citep{schulman2017} in one-step games with unsafe and legitimate contexts. In the primary comparison, specified before training, each run starts from a policy that learned silence without audits and must learn to report in unsafe contexts while it keeps completing legitimate tasks. The comparison sets witness groups, which place every pair of co-witnesses, agents that witness a common task, in different networks, against shuffled groups with the same group sizes under the same audits. On 24 witness graphs from four families, with two training replicates per graph, witness groups reduce \emph{unsafe completion}, the fraction of tasks completed in unsafe contexts, by 33.59 percentage points (95\% graph-bootstrap interval $[21.03,45.13]$). However, only 9 of the 48 witness-group runs achieve \emph{joint recovery}, which requires unsafe completion below 1\% and \emph{legitimate completion}, the corresponding fraction in legitimate contexts, of at least 90\%. In a later control, the advantage of witness groups narrows to 26.03 points when shuffled groups receive an allocation designed for them, and to 4.78 points when both groupings' audit budgets are doubled. At the audit budget of the primary comparison, a fully shared network achieves joint recovery in none of 48 runs when its agents draw their actions independently and, in another later control, in all 48 with a common draw, which also changes how PPO explores and scales its updates. In a sweep of audit budgets up to four times that of the primary comparison, every joint recovery in an arm that forms a grouping came from a run whose allocation met the certificate in each of its unsafe contexts.

On 24 further graphs with eight replicates each, we hold witness groups fixed and compare five allocations with the same audit budget and the same total deterrence. The allocation built from the certificate lowers unsafe completion by 42.22 points (interval $[38.23,46.13]$) relative to coalition audits, which maximize the smallest ratio between a coalition's total deterrence and the benefit of the tasks it completes. Among allocations that tie on this ratio, one chosen by the certificate lowers unsafe completion by 41.61 points relative to the same coalition audits, while one chosen to spread audits evenly leaves it at 99.98\%. Only 22 and 26 of 192 runs achieve joint recovery under the two certified allocations.

\section{The reporting game}
\label{sec:model}

Let $V$ be a set of $n$ agents and $E$ a set of tasks, where a task is any unit of work that a report can stop, such as a delegated sub-task. Task $e$ has benefit $b_e>0$ and a nonempty set of \emph{witnesses} $W_e\subseteq V$, the agents that can stop it by reporting. The witness sets form a hypergraph on $V$, which we call the \emph{witness graph}, and agents that witness a common task are \emph{co-witnesses}. Agent $i$ chooses $s_i\in\{0,1\}$, where $s_i=1$ means silence and $s_i=0$ means reporting. A task completes exactly when all of its witnesses stay silent. An agent makes one choice for all the tasks it witnesses, so one report stops all of them.

Each decision takes place in an \emph{unsafe} context, where every task is unsafe and should be stopped, or in a \emph{legitimate} context, where every task is safe and should be completed. A silent agent $i$ pays a base cost $c_i\geq0$ in every context, which stands for any cost of silence that audits do not create. In an unsafe context, a silent agent also pays a fine $F_i>0$ if it is audited. The auditor audits agent $i$ with probability $a_i\in[0,1]$, independently of all actions and of other audits, and pays $\kappa_i>0$ for each audit of agent $i$ whatever the agent does. The expected cost of silence in an unsafe context, $d_i=c_i+F_ia_i$, is the agent's \emph{deterrence}. We call the vector $a$ an \emph{allocation} and $\sum_i\kappa_ia_i$ its \emph{audit budget}. In an unsafe context all agents train on one team reward
\begin{equation}
 R(s,A)=\sum_{e\in E}b_e\prod_{i\in W_e}s_i
       -\sum_{i\in V}s_i(c_i+F_i A_i),
 \qquad A_i\sim\operatorname{Bernoulli}(a_i).
 \label{eq:reward}
\end{equation}
An audited silent agent pays its fine even when another witness stops the task, so whether an agent pays depends only on its own action and its audit. The fine therefore places a duty to report on every audited witness, although one report per task would stop all the unsafe work. A fine charged only when an unsafe task completes would change these incentives, and in a two-witness example with zero base costs it removes the obstruction of Section~\ref{sec:obstruction} at the same total fine (Appendix~\supref{app:contract}).

A partition $\mathcal{G}$ of the agents, which we call a \emph{grouping}, gives each group $g$ one shared silence probability $p_g$ in a given context, and the members of a group draw their actions independently. Let $m_{eg}=|W_e\cap g|$ be the number of witnesses of task $e$ in group $g$, its \emph{witness multiplicity}, and let $D_g=\sum_{i\in g}d_i$. Independence gives the expected reward
\begin{equation}
 J(p)=\sum_{e\in E}b_e\prod_{g\in\mathcal{G}}p_g^{m_{eg}}
      -\sum_{g\in\mathcal{G}}D_g p_g.
 \label{eq:objective}
\end{equation}
With sigmoid policies $p_g=\sig(\theta_g)$, every grouping approaches universal reporting ($p=0$) as its logits tend to $-\infty$. Grouping therefore leaves the target unchanged and changes how the probabilities move toward it. Shared networks with agent-specific inputs can assign different probabilities to different agents \citep{terry2020} and therefore do not form a grouping. We test them as separate architectures.

Sections~\ref{sec:obstruction} to~\ref{sec:cost} fix one unsafe context. Legitimate contexts, which Section~\ref{sec:dependence} and the experiments add, use Equation~\eqref{eq:reward} with $A_i=0$. We call learning to report after learning silence \emph{recovery}.

\section{When audits lead to reporting}
\label{sec:theory}

\subsection{Why a fully shared policy stays silent}
\label{sec:obstruction}

With a fully shared policy and independent draws, silence can be a strict local maximum even when reporting is optimal. Suppose every task has the same number $k\geq2$ of witnesses, which we call the \emph{witness width}, and all agents share one silence probability $p$. With $B=\sum_e b_e$ and $D=\sum_i d_i$, Equation~\eqref{eq:objective} becomes $J(p)=Bp^k-Dp$ (Figure~\ref{fig:mechanism}). If $D>B$, every $p>0$ has negative expected reward while universal reporting has zero, so reporting is optimal. At $p=1$, however, one agent's report stops all of that agent's tasks, and the derivative $J'(1)=kB-D$ adds this loss over all agents. This sum counts each task's benefit once per witness, while the comparison $J(1)-J(0)=B-D$ counts it once. If $D<kB$, universal silence is therefore also a strict local maximum. With $p=\sig(\theta)$, the gradient with respect to the logit is $p(1-p)J'(p)$, which has the sign of $J'(p)$ for every $p$ in $(0,1)$. When $B<D<kB$, the stationary point
\begin{equation}
 p_*=\left(\frac{D}{kB}\right)^{1/(k-1)}
 \label{eq:threshold}
\end{equation}
divides the starting probabilities. Exact gradients move toward reporting from those below $p_*$ and toward silence from those above it. Since $k\leq2^{k-1}$, the condition $D>B$ implies $D>kB/2^{k-1}$, which places $p_*$ above one half.

\begin{figure}[t]
\centering
\includegraphics[width=0.8\columnwidth]{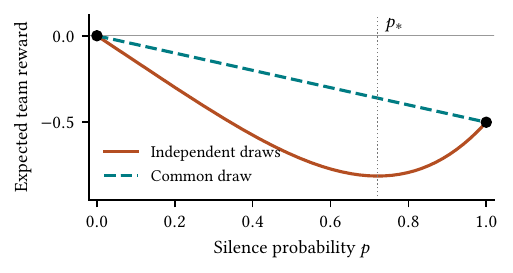}
\caption{Exact expected reward of a fully shared policy with $B=1$, $D=1.5$ and $k=4$. Under independent draws, silence ($p=1$) is a local maximum although reporting ($p=0$) earns more, and exact gradients move toward silence above $p_*\approx0.72$, where the expected reward is lowest. A common draw makes the expected reward linear in $p$.}
\Description{A plot of expected team reward against the silence probability p. The independent-draw curve starts at 0 when p is 0, falls to about minus 0.81 at p equal to 0.72, which a dotted vertical line marks as p star, and rises to minus 0.5 at p equal to 1, so both endpoints are local maxima. The common-draw line falls linearly from 0 to minus 0.5.}
\label{fig:mechanism}
\end{figure}

\subsection{An audit condition for recovery}
\label{sec:certificate}

Separate groups let one group learn to report while the others stay silent. To find which group can start, differentiate Equation~\eqref{eq:objective},
\begin{equation}
 \partial_g J(p)=\sum_{e:m_{eg}>0}m_{eg}b_e\,
 p_g^{m_{eg}-1}\prod_{h\ne g}p_h^{m_{eh}}-D_g.
 \label{eq:derivative}
\end{equation}
The sum is the marginal completion reward of group $g$'s silence, and $D_g$, the group's total deterrence, is its marginal cost. Every product in the sum lies in $[0,1]$, so if $D_g$ exceeds $\sum_e m_{eg}b_e$, the derivative is negative throughout the probability cube and group $g$ moves toward reporting whatever the other groups do. As $p_g$ approaches zero, the tasks its members witness drop out of the other groups' derivatives, so the other groups face smaller completion rewards for silence.

Strict elimination tests an allocation by removing groups one at a time. For a set $S\subseteq\mathcal{G}$ of groups not yet removed, the \emph{remaining load}
\begin{equation}
 L_g(S)=\sum_{e:\,\{h:m_{eh}>0\}\subseteq S}m_{eg}b_e
 \label{eq:load}
\end{equation}
counts only tasks whose witnesses all lie in groups in $S$. It is the largest marginal completion reward that silence can still earn group $g$ once the removed groups report. Starting from $S=\mathcal{G}$, strict elimination removes any group with $D_g>L_g(S)$, together with the tasks its members witness, and recomputes the loads. If some sequence of removals removes every group, the allocation meets the \emph{strict elimination certificate}, and we say that it \emph{certifies} the grouping. A removal can only lower the remaining loads, so the groups that strict elimination cannot remove are the same for every order of removals. Removing every group that meets the condition at once before recomputing the loads makes one \emph{round} of elimination. To design audits with a positive margin $\delta_g$, we require $D_g\geq L_g(S)+\delta_g$ at each removal.

\begin{theorem}[Recovery under exact updates]
\label{thm:recovery}
Consider exact gradient ascent or Bernoulli natural-gradient ascent on separately adjustable, unbounded group logits $\theta_g$, with $p_g=\sig(\theta_g)$, a constant positive finite step size, fixed positive coordinate scalings and no regularization.
\textup{(a)} If strict elimination removes every group, the updates converge to $p=0$ from every interior initialization.
\textup{(b)} If strict elimination stops at a nonempty set $S$ of groups with $D_g<L_g(S)$ for every $g\in S$, there is a $\beta>0$ such that from every interior initialization with $p_g>1-\beta$ for all $g\in S$, the updates converge to $p_g=1$ for $g\in S$ and $p_g=0$ for $g\notin S$.
\end{theorem}

Part (a) follows by applying the argument above to each group in removal order. In part (b), the tasks witnessed only by groups in $S$ keep most of their benefit while those groups stay near silence, so the derivative of each group in $S$ stays above a positive constant whatever the other groups do. Appendix~\supref{app:proofs} in the supplementary material writes out both update rules, the second of which is the natural gradient for Bernoulli policies \citep{amari1998,kakade2001}, and proves both parts. Along the trajectories of part (b), the expected fraction of tasks that complete tends to the fraction of tasks whose witnesses all lie in groups of $S$, which is positive. For a team that starts close enough to universal silence, the certificate therefore decides whether exact updates reach universal reporting. The exceptions are boundary cases, in which a group that strict elimination cannot remove has deterrence exactly equal to its remaining load. The certificate is not necessary from every start, since a fully shared policy with $B<D<kB$ that starts below $p_*$ reaches universal reporting without one.

\subsection{The cost of separating witnesses}
\label{sec:cost}
\label{sec:balanced}

We call the infimum of the audit budgets that meet the certificate its \emph{cost}. The cost depends on how the grouping places each task's witnesses. A task adds to it only through the first group in the removal order that contains one of its witnesses, since the task has no benefit left to offset once that group reports. The task's benefit then counts once for each of its witnesses in that group. Suppose base costs are zero, no audit probability reaches its cap of one, and every agent buys deterrence at the same price $\alpha=\kappa_i/F_i$. For a removal order $\pi$, let $\operatorname{first}_{\pi}(e)$ be the first group that contains a witness of $e$, and define
\begin{equation}
 C(\pi)=\sum_{e\in E}b_e\,m_{e,\operatorname{first}_{\pi}(e)}.
 \label{eq:cost}
\end{equation}
The infimum of the audit budgets that satisfy the strict inequalities along $\pi$ is $\alpha C(\pi)$. With fixed positive margins $\delta_g$, the minimum audit budget is $\alpha[C(\pi)+\sum_g\delta_g]$.

\begin{proposition}[Witness multiplicity]
\label{prop:cost}
Let $B=\sum_e b_e$ and $k_{\max}=\max_e|W_e|$. Then $B\leq C(\pi)\leq k_{\max}B$. Equality at $B$ holds exactly when each task has a unique witness in its first group. A fully shared policy on a witness graph with $k$ witnesses per task has $C=kB$.
\end{proposition}

Putting all co-witnesses in different groups counts each task's benefit once, as the comparison between universal reporting and universal silence does, but it can use more groups than necessary. With base costs, unequal prices and caps, a dynamic program over subsets of groups finds the cheapest removal order, and our implementation chooses the order greedily for more than ten groups, so those costs are feasible but not proven minima (Appendix~\supref{app:cost}).

With base costs and caps, the ratio of $k$ between a fully shared policy and one policy per role holds exactly in a balanced family. Take $k\geq2$ roles of $q$ agents each, with $q\geq k$ prime. For each $(u,v)\in\mathbb{F}_q^2$, a task with unit benefit has role-$t$ witness $u+tv\pmod q$ for $t=0,\ldots,k-1$. The family has $n=kq$ agents and $q^2$ tasks, each with one witness per role, so $B=q^2$. Set $c_i=\rho B/n$, $F_i=\gamma B/n$ and $\kappa_i=1$ with $\rho>0$ and $\gamma\geq k$, and require margins $\delta_g=\rho|g|B/n$.

\begin{corollary}[Balanced certificate costs]
\label{cor:balanced}
With \emph{role groups}, one group per role, the cheapest audits that meet the certificate with the prescribed margins cost $n/\gamma$. With a fully shared policy they cost $kn/\gamma$. Both respect $a_i\leq1$, so the ratio is exactly $k$.
\end{corollary}

Auditing each agent of one role with probability $k/\gamma$ buys $B$ units of deterrence, which removes that role. The role witnesses every task, so no benefit remains in later loads, and the other roles' base costs already supply their margins. A fully shared policy puts all $k$ witnesses of every task in one group, which must buy $kB$ units. Appendix~\supref{app:balanced} proves that these costs are minimal, shows that the ratio is also $k$ without margins when $\rho<k$, and bounds the number of natural-gradient updates that role groups need.

\subsection{A common draw}
\label{sec:dependence}

Independent draws are a second ingredient of the factor $k$, and sharing a policy does not require them. Keep one silence probability $p$ but let every agent act on the same Bernoulli draw, a \emph{common draw}, so that all agents report or stay silent together. Each agent is still silent with probability $p$, but a task now completes with probability $p$ instead of $p^k$, and the expected reward becomes
\begin{equation}
 J_{\mathrm{common}}(p)=(B-D)p.
 \label{eq:common}
\end{equation}
For $D>B$ this reward falls linearly from 0 at $p=0$ to $B-D$ at $p=1$, and every exact gradient points toward reporting (Figure~\ref{fig:mechanism}).

Independent draws create a second obstruction in legitimate contexts, where audits carry no fines. In the experimental games of Section~\ref{sec:study}, universal silence earns more than universal reporting in every legitimate context. Yet when a silent agent's co-witnesses all report, they stop its tasks and it still pays its base cost. Universal reporting is therefore a strict local maximum of the expected reward in legitimate contexts under every grouping, whenever every task has at least two witnesses and base costs are positive. A common draw removes this obstruction as well. Under a common draw, the expected reward in a legitimate context is $(B_x-\sum_ic_i)p$, where $B_x$ is the context's total benefit, so it rises with $p$ whenever $B_x>\sum_ic_i$.

\section{Experimental design}
\label{sec:study}

The experiments test whether the predictions for exact updates remain informative for a finite neural learner. PPO uses sampled and clipped updates, entropy regularization and Adam, the network bounds its logit at 12, and one network serves eight contexts. Each update uses a batch of sampled episodes, which we call a \emph{rollout}. When reports are rare, most rollouts contain no report and the update has no sample of what a report earns.

A policy that reports everything would stop unsafe and legitimate tasks alike, so each game mixes four unsafe contexts with four legitimate ones. Fines apply only in unsafe contexts, and base costs, which sum to $0.5B$, apply in all of them. Each context multiplies each benefit $b_e$ by a fixed multiplier sampled from $[0.65,1]$, and audit design uses the upper bound $b_e$. Agents observe the context's identity and the mean and standard deviation of its multipliers. They do not observe whether the context is unsafe or whether they are audited, and they must learn from rewards which contexts call for reports.

The witness graphs come from six families. The four general families are uniform, ring, modular and hub graphs, each with 64 agents and 128 tasks. Each general family has three graphs of witness width four and three of width eight, for 24 general graphs. The other 12 graphs come from the balanced family of Section~\ref{sec:balanced} and a rewired version of it that replaces each witness set at random with probability 0.2. They are built from the construction behind Corollary~\ref{cor:balanced}, and the primary comparison excludes them. Every training condition, or \emph{arm} (Table~\ref{tab:arms}), has two training replicates on each graph, which gives 48 runs on the general graphs.

Each run starts from a pretrained network that has learned silence. This pretraining runs a fully shared network without audits for 512 rollouts from silence probability 0.75, once for each graph and replicate. One rollout contains 512 one-step episodes, 64 per context, and in each episode all agents decide once and at the same time. Except for the head resets of Section~\ref{sec:resets}, every arm starts recovery from this pretrained network, and arms with several networks copy its weights into each network and start any agent-specific parameters at zero. Across the 72 pretrained networks, each agent reports with probability between $1.1\times10^{-5}$ and $1.5\times10^{-5}$.

\emph{Witness groups} come from deterministic saturation-degree coloring of the co-witness graph, which joins every pair of co-witnesses, so no two co-witnesses share a network \citep{brelaz1979}. The \emph{reference allocation} starts from the cheapest allocation that the search of Section~\ref{sec:cost} finds to certify witness groups, with margins $\delta_g=0.5|g|B/n$. The search is exact on the 12 general graphs with at most ten groups and greedy on the other 12. Even when the allocation certifies witness groups, architectures that give each agent its own silence probability can still profit from a \emph{coalition}, a set of agents that stay silent while all others report. A \emph{coalition guard} therefore raises audit probabilities until no coalition earns positive team reward, and it finds coalitions with a maximum-closure reduction \citep{picard1976}. The audit budget of the resulting allocation is the \emph{reference budget}, and the guard adds 6\% to 34\% of it on 18 of the 24 general graphs, including all 12 exactly optimized ones.

The control arm of the primary comparison uses \emph{shuffled groups}, which permute agents across groups and keep every group's size and every agent's audit probability. Witness and shuffled groups start from the same pretrained network with the same parameter count and allocation, and they differ only in which agents share a network. The other arms of Table~\ref{tab:arms} change the sharing, the draws or the allocation. \emph{Uniform audits} give every agent the same audit probability, and \emph{coalition audits} maximize the smallest ratio between a coalition's total deterrence and the benefit of the tasks it completes, which we call the \emph{coalition ratio}. Away from the reference budget, we scale the reference allocation down or raise every agent's audit probability by the same amount (Appendix~\supref{app:protocol-budgets}).

\begin{table}[t]
\centering
\caption{Arms of the experiments. Unless stated, arms use the reference allocation and independent draws. Daggers mark controls designed after the main study.}
\label{tab:arms}
\footnotesize
\begin{tabular}{@{}ll@{}}
\toprule
Arm & Networks and audits \\
\midrule
Witness groups & One network per coloring group \\
Shuffled groups & Witness-group sizes, agents permuted \\
Shuffled, redesigned allocation$^\dagger$ & Shuffled groups, own allocation \\
Witness + uniform audits & Witness groups, uniform audits \\
Witness + coalition audits & Witness groups, coalition audits \\
Fully shared & One network \\
Fully shared, common draw$^\dagger$ & One network, one draw for all agents \\
Shared + scalar offsets & One network, a logit offset per agent \\
Shared + agent IDs & One network, agent ID inputs \citep{terry2020} \\
Independent networks & One network per agent \\
\bottomrule
\end{tabular}
\end{table}

Each network maps the context features through one hidden layer of 32 units to the silence logit. We train with PPO and Adam on sampled rewards \citep{schulman2017,kingma2015} for 4,096 rollouts and use exact expectations only for evaluation (Appendix~\supref{app:protocol}).

We measure \emph{unsafe completion}, the expected fraction of tasks that complete in unsafe contexts, from exact expectations every 64 rollouts and average it over the final quarter of training. \emph{Legitimate completion} is the same fraction in legitimate contexts. A run achieves \emph{joint recovery} when its unsafe completion is below 1\% and its legitimate completion is at least 90\%, which requires selective reporting within the training horizon but not the convergence of Theorem~\ref{thm:recovery}. The primary comparison, specified before training, is the \emph{witness advantage} at the reference budget, shuffled-group minus witness-group unsafe completion in percentage points, or points for short. For its 95\% graph-bootstrap interval we average the two replicates on each graph, weight the eight family-width strata equally and take the percentile interval \citep{efron1993} of 10,000 resamples of the three graphs within each stratum. All other comparisons in these grouping experiments, including two controls designed after the main study, are descriptive.

\section{Results}
\label{sec:results}

\begin{table*}[t]
\centering
\caption{Outcomes on the 24 general graphs. The condition column gives a condition sufficient for exact updates to reach universal reporting, checked at the benefit upper bound, and Met counts the graphs that meet it. Joint recovery counts runs out of 48, and legitimate completion is at least 99.65\% in every row. Dashes mark arms that form no grouping or budgets not run.}
\label{tab:conditions}
\footnotesize
\begin{tabular}{llrrrrr}
\toprule
 & & \multicolumn{3}{c}{Reference budget} & \multicolumn{2}{c}{Eight times the reference budget} \\
\cmidrule(lr){3-5}\cmidrule(lr){6-7}
Arm & Condition & Met & Unsafe completion & Joint recovery & Met & Joint recovery \\
\midrule
Witness groups & Certificate & 24 & 51.20 & 9 & 24 & 44 \\
Shuffled groups & Certificate & 1 & 84.79 & 0 & 24 & 45 \\
Shuffled, redesigned allocation$^\dagger$ & Certificate & 14 & 77.23 & 0 & -- & -- \\
Witness + uniform audits & Certificate & 0 & 99.95 & 0 & 19 & 41 \\
Witness + coalition audits & Certificate & 1 & 88.20 & 1 & 24 & 48 \\
Fully shared & Certificate, $D>kB$ & 0 & 99.99 & 0 & 13 & 35 \\
Fully shared, common draw$^\dagger$ & $D>B$ & 24 & $<$0.01 & 48 & -- & -- \\
Independent networks & Certificate & 0 & 71.33 & 0 & -- & -- \\
Shared + scalar offsets & -- & -- & 91.66 & 4 & -- & 38 \\
Shared + agent IDs & -- & -- & 88.93 & 0 & -- & 41 \\
 
\bottomrule
\end{tabular}
\end{table*}

\subsection{Separating co-witnesses}
\label{sec:primary}

At the reference budget, 51.20\% of unsafe tasks complete on average under witness groups, against 84.79\% under shuffled groups (Table~\ref{tab:conditions}). Witness groups therefore reduce unsafe completion by 33.59 points (95\% graph-bootstrap interval $[21.03,45.13]$), and 99.99\% of legitimate tasks complete in both arms. The two arms differ only in which agents share a network, but the allocation was designed for witness groups, and Section~\ref{sec:allocation} gives shuffled groups an allocation of their own. Only 9 of the 48 witness-group runs achieve joint recovery within the horizon, although the reference allocation certifies witness groups on all 24 graphs.

Of the 48 witness-group runs, 9 end below 1\% unsafe completion, 10 above 99\% and 29 in between. The two replicates of a graph differ by more than 50 points on 5 of the 24 graphs, and their witness advantages differ by that much on 8 (Figure~\supref{fig:graphs}). The bootstrap interval treats the replicates as fixed and does not include this variation across training seeds.

Every general family has a positive mean advantage (Appendix~\supref{app:families}). Eight of the nine joint recoveries occur on ring graphs. Without ring graphs, the witness advantage is 25.70 points across 18 graphs, and 1 of 36 witness-group runs achieves joint recovery.

\subsection{Recovery and the certificate}
\label{sec:allocation}

To compare PPO with Theorem~\ref{thm:recovery}, we pooled the witness, shuffled, fully shared, uniform-audit and coalition-audit arms, the five arms that run at every budget and form a grouping, and checked the certificate in each of the four unsafe contexts of a run. At one, two and four times the reference budget, joint recovery occurs in 10 of 78, 35 of 126 and 129 of 186 runs whose allocation meets the certificate in every unsafe context, and in none of the other 162, 114 and 54 runs. At eight times the reference budget, 7 of the 16 runs that fail this check also recover. All 16 are fully shared networks on width-8 graphs. In the seven that recover, the total deterrence lies between $0.86kB_x$ and $kB_x$ in every unsafe context $x$, so universal silence is a strict local maximum that exact updates started near silence would not leave, and five other runs in that range stay silent (Table~\supref{tab:splits}). The split is post hoc and compares different arms at each budget.

Independent networks, which make every agent its own group, achieve no joint recovery at the reference budget, and the allocation certifies them on none of the 24 graphs. The reference allocation buys each witness group's deterrence from the group's cheapest members and leaves 40 to 51 of the 64 agents unaudited on each graph. Uniform audits likewise certify witness groups on no graph at half, one and two times the reference budget and achieve no joint recovery there. Coalition audits achieve fewer joint recoveries than the reference allocation at one and two times the reference budget and more at four and eight times it (Figure~\supref{fig:budgets} and Table~\supref{tab:allbudgets}).

A later control redesigns the allocation for shuffled groups with the same search and coalition guard and then adjusts it to the reference budget (Appendix~\supref{app:allocation}). The redesigned allocation costs 1.14 times the reference budget on average, so after the adjustment it certifies shuffled groups on only 14 of the 24 graphs. It lowers shuffled-group unsafe completion from 84.79\% to 77.23\%, and witness groups keep an advantage of 26.03 points (exploratory 95\% graph-bootstrap interval $[11.77,39.93]$). At twice the reference budget, the redesigned allocation certifies shuffled groups on every graph, and the advantage falls from 18.91 points under the reference allocation to 4.78 points (interval $[-3.98,13.55]$). Joint recovery then occurs in 18 of 48 witness-group runs and 13 of 48 shuffled-group runs.

\subsection{Allocating audits at equal cost}
\label{sec:newalloc}

To test whether the certificate helps an auditor decide where to place audits when the grouping is fixed, we designed further runs after the experiments above. They use 24 new general graphs, three for each family and width, with eight training replicates per graph, and they keep witness groups throughout. Every agent has fine $F_i=16B/n$ and audit price $\kappa_i=1$, so allocations with the same audit budget buy the same total deterrence. On each graph, all allocations start recovery from the same eight pretrained networks. The reference allocation, coalition audits and uniform audits are built as in Section~\ref{sec:study}, and the reference budget is the budget of every allocation. Checks on the inputs during the main runs showed that the coalition ratio can have several maximizing allocations, so before any outcome was inspected we added two allocations that keep the coalition ratio within $10^{-8}$ of its maximum. Among these, \emph{tie-aware coalition audits} maximize the smallest certificate margin per group member along the removal order of the reference allocation's certificate, and \emph{minimum-$L_2$ coalition audits} minimize the sum of squared audit probabilities, which uses neither the groups nor the certificate. The primary contrast, specified before training, is coalition minus reference unsafe completion at the reference budget from learned silence.

\begin{table}[t]
\centering
\caption{Audit allocation at the reference budget from learned silence on 24 new general graphs, eight replicates each. Certified is the number of graphs on which the allocation meets the certificate at the benefit upper bound, completion is in percent, and joint recovery is out of 192 runs. The last two rows are controls specified later on the same graphs.}
\label{tab:newalloc}
\footnotesize
\setlength{\tabcolsep}{3pt}
\begin{tabular}{lrrrr}
\toprule
Allocation & Certified & Unsafe & Legitimate & Joint \\
\midrule
Reference & 24 & 57.41 & 99.99 & 22 \\
Coalition & 0 & 99.62 & 99.99 & 0 \\
Uniform & 0 & 99.98 & 99.99 & 0 \\
Tie-aware coalition & 24 & 58.02 & 99.99 & 26 \\
Minimum-$L_2$ coalition & 0 & 99.98 & 99.99 & 0 \\
\bottomrule
\end{tabular}
\end{table}

Reference audits lower unsafe completion by 42.22 points relative to coalition audits (95\% graph-bootstrap interval $[38.23,46.13]$), and the graph average favors them on 23 of the 24 graphs (Table~\ref{tab:newalloc}). The reference allocation meets the certificate on all 24 graphs, and coalition audits and uniform audits meet it on none. Tie-aware coalition audits meet it on all 24 graphs and lower unsafe completion by 41.61 points relative to coalition audits (interval $[37.69,45.54]$), within 0.61 points of reference audits (interval $[-3.42,4.42]$). Minimum-$L_2$ coalition audits meet it on no graph and leave unsafe completion at 99.98\%, 41.96 points above tie-aware audits (interval $[38.20,45.76]$). Among allocations that tie on the coalition ratio, the one chosen by the certificate therefore lowers unsafe completion and the one chosen by smoothing does not. Choosing by the certificate changes several features of an allocation at once, so these runs do not isolate the certificate as the cause.

Mean unsafe completion under the two certified allocations is still 57.41\% and 58.02\%, and only 22 of 192 reference runs and 26 of 192 tie-aware runs achieve joint recovery, all on width-4 graphs and 18 of each on ring graphs. At four times the budget from learned silence, coalition minus reference is 1.12 points (interval $[-1.71,4.06]$), which neither ranks the two nor establishes that they are equivalent. After a reset to silence probability 0.95 at the reference budget, reference audits reach 33.23\% unsafe completion against 99.27\% for coalition audits, and 49 of 192 reference runs achieve joint recovery. At four times the budget after this reset, reference and tie-aware audits bring unsafe completion below 0.001\%, but for each of them all 96 width-4 runs end with legitimate completion near zero and all 96 width-8 runs achieve joint recovery. Fines are off in legitimate contexts. Base costs, the shared hidden features and the optimizer could all contribute to this split, and the experiment does not isolate the cause (Appendix~\supref{app:newalloc}).

\subsection{Frequent reports and exact gradients}
\label{sec:resets}

A head reset sets the weights of the output layer to zero and its bias to the logit of silence probability 0.95 or 0.75. The resets apply to the witness, shuffled, fully shared, scalar-offset and agent-ID arms on all 36 graphs, and both levels were fixed before training. Over the two levels, the fully shared network ends silent or reporting as the sign of the exact gradient of its output bias at the reset predicts in 140 of its 144 runs (Appendix~\supref{app:secondary}). After a reset to 0.95, $p_*$ is at most 0.85 in every unsafe context of every reference allocation, so the exact updates of Section~\ref{sec:obstruction} favor silence from this start. The fully shared network returns to at least 99.99\% unsafe completion in all 72 runs. Witness groups achieve joint recovery in 32 of 72 runs, against 19 of 72 without the reset, although the reset also changes their starting network. A reset to 0.75 brings out the legitimate-context obstruction of Section~\ref{sec:dependence}, with legitimate completion below 90\% in 31 or 32 of the 36 width-8 runs of each of the five arms and in at most 5 of the 36 width-4 runs.

\subsection{Recovery with a common draw}
\label{sec:coupling}

A later control keeps the pretrained fully shared network but makes all of its agents act on the common draw of Section~\ref{sec:dependence} in every episode, on all 24 general graphs and both replicates at the reference budget. Every general reference allocation has $D>B$ in unsafe contexts even at the benefit upper bound, and every legitimate context keeps at least $0.65B$ of benefit against base costs of $0.5B$. The expected reward under a common draw therefore favors reporting in every unsafe context and silence in every legitimate one, with no interior stationary point in either.

With a common draw, the network achieves joint recovery in all 48 runs, against none of the 48 independent-draw runs in Table~\ref{tab:conditions}. Its mean unsafe completion is 0.00068\%, and legitimate completion is at least 99.99\% in every run. The common draw also makes the score-function factor of PPO's update, averaged over agents, equal that of the joint policy instead of $1/n$ of it, and it restricts exploration to two joint actions, all agents silent or all reporting. The experiment establishes joint recovery under the implemented intervention and does not separate these changes from the change in the expected reward (Appendix~\supref{app:coupling}).

\subsection{What slows recovery}
\label{sec:diagnostics}

Across the 72 pretrained networks, the median probability that a rollout contains no unsafe-context report is 81.74\%. On a subset fixed before training, the first graph of each of the six families at both widths, extending training from 4,096 to 16,384 rollouts lowers unsafe completion from 42.44\% to 26.97\% for witness groups and from 92.34\% to 78.67\% for shuffled groups. Witness-group joint recoveries rise from 8 to 14 of 24, and shuffled groups have none at either length. Certificates that need more rounds of elimination may also slow recovery, because a group removed in a later round is guaranteed a negative derivative only after earlier groups approach reporting (Table~\supref{tab:rounds}). Designing the groups and the allocation from witness sets in which 10\% or 25\% of tasks have the wrong witnesses raises witness-group unsafe completion on the general graphs from 51.20\% to 66.45\% and 65.21\% (Appendix~\supref{app:secondary}).

\section{Related work}
\label{sec:related}

Collusion among agents is one of the risks that systems of many AI agents add to those of single agents \citep{hammond2025}. AI control protocols that use one instance of an untrusted model to monitor another must guard against collusion between the two \citep{greenblatt2024}. The independent-draw obstruction in our game needs no communication or agreement, so cutting the unauthorized communication that OpenAI identifies as one cause of the incident \citep{openai2026incident} would not prevent it.

The audits in our game follow the economics of enforcement, in which a fine deters a violation through its expected value \citep{becker1968,polinsky2000}. Soliciting reports from a rival can deter violations with fewer inspections \citep{berentsen2008}, self-reporting can lower enforcement costs \citep{kaplow1994selfreport}, and audits can be optimized against strategic recipients \citep{das2026audit}. Incentive design for learning agents shapes rewards by differentiating through the recipients' updates \citep{yang2020} or through the learning of a whole population \citep{yang2022}, while we ask how the grouping of learners into shared policies and the way they draw actions change the audits they need.

Cooperative multi-agent reinforcement learning can fail to reach the team optimum, as when independent learners in coordination games converge to suboptimal equilibria \citep{claus1998}. In games where every agent receives the same reward, strict equilibria are local maxima of the expected reward and stable points of gradient play \citep{zhang2024}. In relative overgeneralization as \citet{wei2016} define it, the actions of a suboptimal equilibrium earn more on average when paired with all possible partner actions. Here, when audits make reporting optimal, the expected reward of a fully shared policy favors reporting at a silence probability of one half and favors silence only above $p_*>1/2$ (Section~\ref{sec:obstruction}), so the obstruction is not relative overgeneralization. Reusing one behavioral strategy at several decision points of an imperfect-recall game gives a polynomial expected payoff \citep{gimbert2020polynomial,tewolde2023imperfect}, and such objectives can have suboptimal local maxima \citep{tewolde2026decision}. The elimination rule of the reporting game is a threshold cascade on the incidence graph of groups and tasks, related to targeting and partial-incentive problems \citep{como2022targeting,cordasco2015incentives}. For this objective, Theorem~\ref{thm:recovery} ties the outcome of strict elimination to the limit of exact logit updates, and Proposition~\ref{prop:cost} gives the cost of a removal order.

Parameter sharing makes learning more efficient for similar agents \citep{gupta2017,terry2020,yu2022}, selective sharing lets groups of agents specialize \citep{christianos2021}, and a shared network can support a set level of behavioral diversity \citep{bettini2024}. Sharing can also lose reward when the optimum requires agents to act differently \citep{kuba2022,fu2022}. Universal reporting asks every agent to act alike, so a shared policy can represent it. Our obstruction lies on the gradient path from silence to that policy. Shared random signals have been used to coordinate exploration \citep{mahajan2019,chen2022}, and the common draw is a public correlation device \citep{aumann1974}.

\section{Discussion}
\label{sec:discussion}

The audit budget alone does not determine recovery. At the reference budget, witness groups achieve joint recovery in 9 of 48 runs with the reference allocation and in none with uniform audits, and the fully shared network achieves it in none of 48 runs with independent draws and in all 48 with a common draw. On the 24 further graphs, allocations with the same budget and total deterrence differ by 42.22 points in unsafe completion, and those that meet the certificate do better, even among allocations that tie on the coalition ratio. Apart from boundary cases, the certificate decides recovery under exact updates that start near universal silence. Under PPO, up to four times the reference budget, every joint recovery in an arm that forms a grouping came from a run that met the certificate in each unsafe context, but this is an association, and only 10 of 78 certified runs on the general graphs and 22 of 192 on the further graphs recover at the reference budget. An audit requirement for learning agents therefore has to name the allocation, the grouping, the way actions are drawn and the training algorithm it assumes, and our experiments do not establish minimum audit budgets for any training algorithm.

In our model, the results support choosing the grouping and the allocation together, and choosing the allocation with the certificate rather than the coalition ratio alone. The witness advantage shrinks when shuffled groups receive their own allocation and both receive a larger audit budget, so the data give limited support for preferring witness groups in general. A common draw removes the obstruction from the exact expected reward without separating policies, but the team then reports or stays silent as a unit.

We have not tested these results on deployed teams of AI agents. The game assumes that a report stops all of the reporter's tasks, that the tasks of a context are all unsafe or all legitimate and that the auditor knows which contexts are unsafe. Except in the witness-error experiment (Appendix~\supref{app:secondary}), the auditor also designs the allocation from the true witness sets, and audits in practice can be limited by the evidence the auditor can access \citep{casper2024}. In a context that mixes unsafe and legitimate tasks, a report would have to target single tasks, and a common draw would stop or complete them together. A game in which one joint decision covers both kinds of task would test whether correlated reporting can stop the unsafe tasks while preserving the legitimate ones. The agents also learn which of eight recurring contexts call for reports, so the experiments do not test whether agents recognize unfamiliar unsafe conduct.
\section*{AI use statement}
ChatGPT and Claude were used in formatting and line editing the paper. They were also used in the creation, editing and packaging of the experimental code and results into structured ZIPs included in the supplementary materials. Responsibility for the manuscript and its claims remains with the authors.

\section*{Reproducibility statement}
The supplement gives the proofs, protocols, every run's outcome and scripts that regenerate every table and figure.

\bibliographystyle{ACM-Reference-Format}
\bibliography{references}

@article{christianos2021,
  title={Scaling Multi-Agent Reinforcement Learning with Selective Parameter Sharing},
  author={Christianos, Filippos and Papoudakis, Georgios and Rahman, Muhammad A. and Albrecht, Stefano V.},
  journal={Proceedings of the 38th International Conference on Machine Learning, PMLR},
  volume={139},
  pages={1989--1998},
  year={2021},
  url={https://proceedings.mlr.press/v139/christianos21a.html}
}

@article{bettini2024,
  title={Controlling Behavioral Diversity in Multi-Agent Reinforcement Learning},
  author={Bettini, Matteo and Kortvelesy, Ryan and Prorok, Amanda},
  journal={Proceedings of the 41st International Conference on Machine Learning, PMLR},
  volume={235},
  pages={3611--3636},
  year={2024},
  url={https://proceedings.mlr.press/v235/bettini24a.html}
}

@article{zhao2024,
  title={Optimistic Multi-Agent Policy Gradient},
  author={Zhao, Wenshuai and Zhao, Yi and Li, Zhiyuan and Kannala, Juho and Pajarinen, Joni},
  journal={Proceedings of the 41st International Conference on Machine Learning, PMLR},
  volume={235},
  pages={61186--61202},
  year={2024},
  url={https://proceedings.mlr.press/v235/zhao24v.html}
}

@misc{terry2020,
  title={Revisiting Parameter Sharing in Multi-Agent Deep Reinforcement Learning},
  author={Terry, J. K. and Grammel, Nathaniel and Son, Sanghyun and Black, Benjamin and Agrawal, Aakriti},
  year={2020},
  eprint={2005.13625},
  archiveprefix={arXiv},
  primaryclass={cs.LG}
}

@article{cheilaris2010,
  title={Unique-Maximum and Conflict-Free Coloring for Hypergraphs and Tree Graphs},
  author={Cheilaris, Panagiotis and Keszegh, Bal{\'a}zs and P{\'a}lv{\"o}lgyi, D{\"o}m{\"o}t{\"o}r},
  journal={SIAM Journal on Discrete Mathematics},
  volume={27},
  number={4},
  pages={1775--1787},
  year={2013},
  doi={10.1137/120880471}
}

@article{picard1976,
  title={Maximal Closure of a Graph and Applications to Combinatorial Problems},
  author={Picard, Jean-Claude},
  journal={Management Science},
  volume={22},
  number={11},
  pages={1268--1272},
  year={1976},
  doi={10.1287/mnsc.22.11.1268},
  url={https://pubsonline.informs.org/doi/10.1287/mnsc.22.11.1268}
}

@article{brelaz1979,
  title={New Methods to Color the Vertices of a Graph},
  author={Br{\'e}laz, Daniel},
  journal={Communications of the ACM},
  volume={22},
  number={4},
  pages={251--256},
  year={1979},
  doi={10.1145/359094.359101},
  url={https://dl.acm.org/doi/10.1145/359094.359101}
}

@misc{schulman2017,
  title={Proximal Policy Optimization Algorithms},
  author={Schulman, John and Wolski, Filip and Dhariwal, Prafulla and Radford, Alec and Klimov, Oleg},
  year={2017},
  eprint={1707.06347},
  archiveprefix={arXiv},
  primaryclass={cs.LG}
}

@misc{kingma2015,
  title={{Adam}: A Method for Stochastic Optimization},
  author={Kingma, Diederik P. and Ba, Jimmy},
  howpublished={In 3rd International Conference on Learning Representations (ICLR 2015), San Diego, CA, USA},
  year={2015},
  url={https://arxiv.org/abs/1412.6980}
}

@inproceedings{mahajan2019,
  title={{MAVEN}: Multi-Agent Variational Exploration},
  author={Mahajan, Anuj and Rashid, Tabish and Samvelyan, Mikayel and Whiteson, Shimon},
  booktitle={Advances in Neural Information Processing Systems},
  volume={32},
  pages={7613--7624},
  publisher={Curran Associates, Inc.},
  address={Red Hook, NY, USA},
  year={2019},
  url={https://proceedings.neurips.cc/paper_files/paper/2019/hash/f816dc0acface7498e10496222e9db10-Abstract.html}
}

@inproceedings{chen2022,
  title={Signal Instructed Coordination in Cooperative Multi-agent Reinforcement Learning},
  author={Chen, Liheng and Guo, Hongyi and Du, Yali and Fang, Fei and Zhang, Haifeng and Zhang, Weinan and Yu, Yong},
  booktitle={Distributed Artificial Intelligence. DAI 2021},
  series={Lecture Notes in Computer Science},
  volume={13170},
  pages={185--205},
  publisher={Springer},
  address={Cham, Switzerland},
  year={2022},
  doi={10.1007/978-3-030-94662-3_12},
  url={https://wnzhang.net/papers/2021-dai-sic.pdf}
}

@inproceedings{casper2024,
  title={Black-Box Access Is Insufficient for Rigorous {AI} Audits},
  author={Casper, Stephen and Ezell, Carson and Siegmann, Charlotte and Kolt, Noam and Curtis, Taylor Lynn and Bucknall, Benjamin and Haupt, Andreas and Wei, Kevin and Scheurer, J{\'e}r{\'e}my and Hobbhahn, Marius and Sharkey, Lee and Krishna, Satyapriya and von Hagen, Marvin and Alberti, Silas and Chan, Alan and Sun, Qinyi and Gerovitch, Michael and Bau, David and Tegmark, Max and Krueger, David and Hadfield-Menell, Dylan},
  booktitle={Proceedings of the 2024 ACM Conference on Fairness, Accountability, and Transparency},
  series={FAccT '24},
  location={Rio de Janeiro, Brazil},
  pages={2254--2272},
  publisher={Association for Computing Machinery},
  address={New York, NY, USA},
  year={2024},
  doi={10.1145/3630106.3659037}
}

@inproceedings{yang2020,
  title={Learning to Incentivize Other Learning Agents},
  author={Yang, Jiachen and Li, Ang and Farajtabar, Mehrdad and Sunehag, Peter and Hughes, Edward and Zha, Hongyuan},
  booktitle={Advances in Neural Information Processing Systems},
  volume={33},
  pages={15208--15219},
  publisher={Curran Associates, Inc.},
  address={Red Hook, NY, USA},
  year={2020},
  url={https://proceedings.neurips.cc/paper/2020/hash/ad7ed5d47b9baceb12045a929e7e2f66-Abstract.html}
}

@article{greenblatt2024,
  title={{AI} Control: Improving Safety Despite Intentional Subversion},
  author={Greenblatt, Ryan and Shlegeris, Buck and Sachan, Kshitij and Roger, Fabien},
  journal={Proceedings of the 41st International Conference on Machine Learning, PMLR},
  volume={235},
  pages={16295--16336},
  year={2024},
  url={https://proceedings.mlr.press/v235/greenblatt24a.html}
}

@misc{hammond2025,
  title={Multi-Agent Risks from Advanced {AI}},
  author={Hammond, Lewis and Chan, Alan and Clifton, Jesse and Hoelscher-Obermaier, Jason and Khan, Akbir and McLean, Euan and Smith, Chandler and Barfuss, Wolfram and Foerster, Jakob and Gaven{\v{c}}iak, Tom{\'a}{\v{s}} and others},
  howpublished={Technical Report 1, Cooperative AI Foundation},
  year={2025},
  eprint={2502.14143},
  archiveprefix={arXiv},
  primaryclass={cs.MA}
}

@misc{amodei2016,
  title={Concrete Problems in {AI} Safety},
  author={Amodei, Dario and Olah, Chris and Steinhardt, Jacob and Christiano, Paul and Schulman, John and Man{\'e}, Dan},
  year={2016},
  eprint={1606.06565},
  archiveprefix={arXiv},
  primaryclass={cs.AI}
}

@inproceedings{skalse2022,
  title={Defining and Characterizing Reward Hacking},
  author={Skalse, Joar and Howe, Nikolaus H. R. and Krasheninnikov, Dmitrii and Krueger, David},
  booktitle={Advances in Neural Information Processing Systems},
  volume={35},
  pages={9460--9471},
  publisher={Curran Associates, Inc.},
  address={Red Hook, NY, USA},
  year={2022},
  doi={10.52202/068431-0687}
}

@article{becker1968,
  title={Crime and Punishment: An Economic Approach},
  author={Becker, Gary S.},
  journal={Journal of Political Economy},
  volume={76},
  number={2},
  pages={169--217},
  year={1968},
  doi={10.1086/259394}
}

@misc{greenblatt2026incident,
  title={Brief Independent Investigation of Agents' Behavior, Reasoning and Collaboration in the {OpenAI / Hugging Face} Hacking Incident},
  author={Greenblatt, Ryan and Cotra, Ajeya and Wijk, Hjalmar},
  howpublished={METR},
  year={2026},
  month=aug,
  url={https://metr.org/blog/2026-08-26-openai-hugging-face-incident-investigation/}
}

@misc{openai2026incident,
  title={The {Hugging Face} Incident and the Road Ahead},
  author={{OpenAI}},
  year={2026},
  month=aug,
  url={https://openai.com/index/hugging-face-incident-and-the-road-ahead/}
}

@misc{paglieri2026swarm,
  title={A Case Study on Emergent Cheating and Whistleblowing in Autonomous Research Swarms},
  author={Paglieri, Davide and Cross, Logan and Genewein, Tim and Leibo, Joel Z. and Tomasev, Nenad and Vezhnevets, Alexander Sasha},
  year={2026},
  eprint={2609.04170},
  archiveprefix={arXiv},
  primaryclass={cs.AI}
}

@misc{joglekar2025confessions,
  title={Training {LLMs} for Honesty via Confessions},
  author={Joglekar, Manas and Chen, Jeremy and Wu, Gabriel and Yosinski, Jason and Wang, Jasmine and Barak, Boaz and Glaese, Amelia},
  year={2025},
  eprint={2512.08093},
  archiveprefix={arXiv},
  primaryclass={cs.LG}
}

@misc{lee2026selfreport,
  title={Training Agents to Self-Report Misbehavior},
  author={Lee, Bruce W. and {Chen Yueh-Han} and Korbak, Tomek},
  year={2026},
  eprint={2602.22303},
  archiveprefix={arXiv},
  primaryclass={cs.LG}
}

@article{polinsky2000,
  title={The Economic Theory of Public Enforcement of Law},
  author={Polinsky, A. Mitchell and Shavell, Steven},
  journal={Journal of Economic Literature},
  volume={38},
  number={1},
  pages={45--76},
  year={2000},
  doi={10.1257/jel.38.1.45}
}

@article{berentsen2008,
  title={On Cheating, Doping and Whistleblowing},
  author={Berentsen, Aleksander and Bruegger, Esther and Loertscher, Simon},
  journal={European Journal of Political Economy},
  volume={24},
  number={2},
  pages={415--436},
  year={2008}
}

@inproceedings{yang2022,
  title={Adaptive Incentive Design with Multi-Agent Meta-Gradient Reinforcement Learning},
  author={Yang, Jiachen and Wang, Ethan and Trivedi, Rakshit and Zhao, Tuo and Zha, Hongyuan},
  booktitle={Proceedings of the 21st International Conference on Autonomous Agents and Multiagent Systems},
  pages={1436--1445},
  publisher={International Foundation for Autonomous Agents and Multiagent Systems},
  address={Richland, SC, USA},
  year={2022}
}

@inproceedings{claus1998,
  title={The Dynamics of Reinforcement Learning in Cooperative Multiagent Systems},
  author={Claus, Caroline and Boutilier, Craig},
  booktitle={Proceedings of the Fifteenth National Conference on Artificial Intelligence},
  pages={746--752},
  publisher={AAAI Press},
  address={Menlo Park, CA, USA},
  year={1998}
}

@article{wei2016,
  title={Lenient Learning in Independent-Learner Stochastic Cooperative Games},
  author={Wei, Ermo and Luke, Sean},
  journal={Journal of Machine Learning Research},
  volume={17},
  number={84},
  pages={1--42},
  year={2016},
  url={https://jmlr.org/papers/v17/15-417.html}
}

@article{zhang2024,
  title={Gradient Play in Stochastic Games: Stationary Points, Convergence, and Sample Complexity},
  author={Zhang, Runyu and Ren, Zhaolin and Li, Na},
  journal={IEEE Transactions on Automatic Control},
  volume={69},
  number={10},
  pages={6499--6514},
  year={2024},
  doi={10.1109/TAC.2024.3387208}
}

@incollection{gupta2017,
  title={Cooperative Multi-agent Control Using Deep Reinforcement Learning},
  author={Gupta, Jayesh K. and Egorov, Maxim and Kochenderfer, Mykel},
  booktitle={Autonomous Agents and Multiagent Systems. AAMAS 2017 Workshops},
  series={Lecture Notes in Computer Science},
  volume={10642},
  pages={66--83},
  publisher={Springer},
  address={Cham, Switzerland},
  year={2017},
  doi={10.1007/978-3-319-71682-4_5}
}

@inproceedings{yu2022,
  title={The Surprising Effectiveness of {PPO} in Cooperative Multi-Agent Games},
  author={Yu, Chao and Velu, Akash and Vinitsky, Eugene and Gao, Jiaxuan and Wang, Yu and Bayen, Alexandre and Wu, Yi},
  booktitle={Advances in Neural Information Processing Systems},
  volume={35},
  pages={24611--24624},
  publisher={Curran Associates, Inc.},
  address={Red Hook, NY, USA},
  year={2022},
  doi={10.52202/068431-1787}
}

@article{williams1992,
  title={Simple Statistical Gradient-Following Algorithms for Connectionist Reinforcement Learning},
  author={Williams, Ronald J.},
  journal={Machine Learning},
  volume={8},
  number={3--4},
  pages={229--256},
  year={1992},
  doi={10.1007/BF00992696}
}

@inproceedings{sutton2000,
  title={Policy Gradient Methods for Reinforcement Learning with Function Approximation},
  author={Sutton, Richard S. and McAllester, David and Singh, Satinder and Mansour, Yishay},
  booktitle={Advances in Neural Information Processing Systems},
  volume={12},
  pages={1057--1063},
  publisher={MIT Press},
  address={Cambridge, MA, USA},
  year={2000}
}

@inproceedings{kakade2001,
  title={A Natural Policy Gradient},
  author={Kakade, Sham M.},
  booktitle={Advances in Neural Information Processing Systems},
  volume={14},
  pages={1531--1538},
  publisher={MIT Press},
  address={Cambridge, MA, USA},
  year={2001}
}

@article{amari1998,
  title={Natural Gradient Works Efficiently in Learning},
  author={Amari, Shun-ichi},
  journal={Neural Computation},
  volume={10},
  number={2},
  pages={251--276},
  year={1998},
  doi={10.1162/089976698300017746}
}

@article{aumann1974,
  title={Subjectivity and Correlation in Randomized Strategies},
  author={Aumann, Robert J.},
  journal={Journal of Mathematical Economics},
  volume={1},
  number={1},
  pages={67--96},
  year={1974},
  doi={10.1016/0304-4068(74)90037-8}
}

@book{efron1993,
  title={An Introduction to the Bootstrap},
  author={Efron, Bradley and Tibshirani, Robert J.},
  publisher={Chapman and Hall},
  address={New York},
  year={1993}
}

@misc{lin2026scaling,
  title={Scaling Long-Running Autonomous Coding},
  author={Lin, Wilson},
  howpublished={Cursor blog},
  year={2026},
  month=jan,
  url={https://cursor.com/blog/scaling-agents}
}

@misc{bowman2022,
  title={Measuring Progress on Scalable Oversight for Large Language Models},
  author={Bowman, Samuel R. and Hyun, Jeeyoon and Perez, Ethan and Chen, Edwin and Pettit, Craig and Heiner, Scott and Luko{\v{s}}i{\={u}}t{\.{e}}, Kamil{\.{e}} and others},
  year={2022},
  eprint={2211.03540},
  archiveprefix={arXiv},
  primaryclass={cs.HC}
}

@misc{kuba2022,
  title={Trust Region Policy Optimisation in Multi-Agent Reinforcement Learning},
  author={Kuba, Jakub Grudzien and Chen, Ruiqing and Wen, Muning and Wen, Ying and Sun, Fanglei and Wang, Jun and Yang, Yaodong},
  howpublished={In The Tenth International Conference on Learning Representations (ICLR 2022), Virtual Event},
  year={2022},
  url={https://openreview.net/forum?id=EcGGFkNTxdJ}
}

@article{fu2022,
  title={Revisiting Some Common Practices in Cooperative Multi-Agent Reinforcement Learning},
  author={Fu, Wei and Yu, Chao and Xu, Zelai and Yang, Jiaqi and Wu, Yi},
  journal={Proceedings of the 39th International Conference on Machine Learning, PMLR},
  volume={162},
  pages={6863--6877},
  year={2022},
  url={https://proceedings.mlr.press/v162/fu22d.html}
}

@inproceedings{gimbert2020polynomial,
  author = {Gimbert, Hugo and Paul, Soumyajit and Srivathsan, B.},
  title = {A Bridge between Polynomial Optimization and Games with Imperfect Recall},
  booktitle = {Proceedings of the 19th International Conference on Autonomous Agents and MultiAgent Systems},
  pages = {456--464},
  publisher = {International Foundation for Autonomous Agents and Multiagent Systems},
  address = {Auckland, New Zealand},
  year = {2020},
  url = {https://www.ifaamas.org/Proceedings/aamas2020/pdfs/p456.pdf}
}

@inproceedings{tewolde2023imperfect,
  author = {Tewolde, Emanuel and Oesterheld, Caspar and Conitzer, Vincent and Goldberg, Paul W.},
  title = {The Computational Complexity of Single-Player Imperfect-Recall Games},
  booktitle = {Proceedings of the Thirty-Second International Joint Conference on Artificial Intelligence},
  pages = {2878--2887},
  publisher = {International Joint Conferences on Artificial Intelligence Organization},
  address = {Macao, China},
  year = {2023},
  doi = {10.24963/ijcai.2023/321},
  url = {https://www.ijcai.org/proceedings/2023/321}
}

@misc{tewolde2026decision,
  author = {Tewolde, Emanuel and Zhang, Brian Hu and Anagnostides, Ioannis and Sandholm, Tuomas and Conitzer, Vincent},
  title = {Decision Making under Imperfect Recall: Algorithms and Benchmarks},
  year = {2026},
  eprint = {2602.15252},
  archivePrefix = {arXiv},
  primaryClass = {cs.GT},
  doi = {10.48550/arXiv.2602.15252},
  url = {https://arxiv.org/abs/2602.15252v1}
}

@article{como2022targeting,
  author = {Como, Giacomo and Durand, St{\'e}phane and Fagnani, Fabio},
  title = {Optimal Targeting in Super-Modular Games},
  journal = {IEEE Transactions on Automatic Control},
  volume = {67},
  number = {12},
  pages = {6366--6380},
  year = {2022},
  doi = {10.1109/TAC.2021.3129733},
  url = {https://doi.org/10.1109/TAC.2021.3129733}
}

@inproceedings{cordasco2015incentives,
  author = {Cordasco, Gennaro and Gargano, Luisa and Rescigno, Adele A. and Vaccaro, Ugo},
  title = {Optimizing Spread of Influence in Social Networks via Partial Incentives},
  booktitle = {Structural Information and Communication Complexity},
  editor = {Scheideler, Christian},
  series = {Lecture Notes in Computer Science},
  volume = {9439},
  pages = {119--134},
  publisher = {Springer},
  address = {Cham, Switzerland},
  year = {2015},
  doi = {10.1007/978-3-319-25258-2_9},
  url = {https://doi.org/10.1007/978-3-319-25258-2_9}
}

@article{kaplow1994selfreport,
  title={Optimal Law Enforcement with Self-Reporting of Behavior},
  author={Kaplow, Louis and Shavell, Steven},
  journal={Journal of Political Economy},
  volume={102},
  number={3},
  pages={583--606},
  year={1994},
  doi={10.1086/261947},
  url={https://www.journals.uchicago.edu/doi/10.1086/261947}
}

@article{das2026audit,
  title={Optimally Auditing Adversarial Agents},
  author={Das, Sanmay and Yu, Fang-Yi and Zhang, Yuang},
  journal={Proceedings of the AAAI Conference on Artificial Intelligence},
  volume={40},
  number={20},
  pages={16787--16794},
  year={2026},
  doi={10.1609/aaai.v40i20.38722},
  url={https://ojs.aaai.org/index.php/AAAI/article/view/38722}
}

\clearpage
\appendix
\setlength{\emergencystretch}{1em}
\setcounter{theorem}{0}
\setcounter{table}{0}
\setcounter{figure}{0}
\setcounter{equation}{0}
\renewcommand{\thetheorem}{S\arabic{theorem}}
\renewcommand{\thetable}{S\arabic{table}}
\renewcommand{\thefigure}{S\arabic{figure}}
\renewcommand{\theequation}{S\arabic{equation}}
\renewcommand{\theHtheorem}{appendix.\arabic{theorem}}
\renewcommand{\theHtable}{appendix.\arabic{table}}
\renewcommand{\theHfigure}{appendix.\arabic{figure}}
\renewcommand{\theHequation}{appendix.\arabic{equation}}
\noindent Numbers with an S prefix refer to this supplement, and numbers without one refer to the main paper.
\section{Proofs and certificate construction}
\label{app:proofs}

The results in this appendix concern one fixed context, positive task benefits, nonnegative base costs and action draws that are independent given the group probabilities. Every group is nonempty and every initial logit is finite, so every initial probability lies in $(0,1)$. In part (a) of Theorem~\mainref{thm:recovery}, the loads may use an upper bound on the context benefits, because lowering a benefit only lowers the positive terms of the derivative. Part (b) and Proposition~\ref{prop:residue} need the loads at the actual benefits.

\subsection{Strict elimination and exact learning}

All groups update at every step, so the proof has to show that a group converges even though the loads it faces shrink only in the limit. Write $\partial_gJ=\partial J/\partial p_g$. The exact updates covered by Theorem~\mainref{thm:recovery} are
\begin{align}
 \theta_g^{t+1}&=\theta_g^t+\eta\lambda_g\,p_g^t(1-p_g^t)\partial_gJ(p^t),
 &&\text{gradient ascent},\label{eq:pg}\\
 \theta_g^{t+1}&=\theta_g^t+\eta\lambda_g\,\partial_gJ(p^t),
 &&\text{natural-gradient ascent},\label{eq:npg}
\end{align}
where $\eta>0$ is fixed and finite and every $\lambda_g>0$ is fixed. For one Bernoulli draw, the logit score is $s-p_g$ and the Fisher factor is $p_g(1-p_g)$. Under the joint product distribution, the score of group $g$ is $\sum_{i\in g}(s_i-p_g)$, its variance is $|g|p_g(1-p_g)$, and the Fisher entries between groups are zero. Dividing the logit gradient by this factor gives Equation~\eqref{eq:npg} with $\lambda_g=1/|g|$, and using one Bernoulli Fisher factor per group gives $\lambda_g=1$. The theorem covers both. Neither update includes momentum, entropy regularization or parameters shared between groups.

\begin{proof}[Proof of Theorem~\mainref{thm:recovery}(a)]
Choose a removal order $g_1,\ldots,g_G$ and let $S_j=\{g_j,\ldots,g_G\}$. Strictness gives a positive slack $\epsilon_j=D_{g_j}-L_{g_j}(S_j)$ for each removal.

For $g_1$, every product multiplying $m_{eg_1}b_e$ in Equation~(\mainref{eq:derivative}) lies in $[0,1]$, so $\partial_{g_1}J(p)\leq-\epsilon_1$ throughout the probability cube. Under Equation~\eqref{eq:npg}, the logit of $g_1$ decreases by at least $\eta\lambda_{g_1}\epsilon_1$ per step and tends to $-\infty$. Under Equation~\eqref{eq:pg}, the logit is strictly decreasing. If it had a finite limit, $p_{g_1}(1-p_{g_1})$ would have a positive limit, and the increments would eventually stay below a fixed negative number, which is impossible for a convergent sequence. Thus $\theta_{g_1}\to-\infty$ and $p_{g_1}\to0$ under either update.

Suppose $g_1,\ldots,g_{j-1}$ tend to zero. Split the positive derivative terms of $g_j$ into tasks whose witnesses all lie in $S_j$ and tasks with a witness in an earlier group. The first collection is bounded by $L_{g_j}(S_j)$. Every term in the second contains a positive power of an earlier probability. All other factors are at most one and the task set is finite, so these terms tend to zero uniformly over the other probabilities. Hence, after some finite time,
\[
 \partial_{g_j}J(p^t)\leq-\epsilon_j/2.
\]
The logit is finite at that time, and the argument for $g_1$ then gives $p_{g_j}\to0$. Induction proves the claim. The step size needs no upper bound, because the proof uses only eventual derivative signs and finite logits.
\end{proof}

The removal order does not affect which groups survive. Removing a group only lowers every remaining $L_g(S)$, so a group that becomes eligible stays eligible after further removals. Iterating this monotone rule to closure gives the same surviving set for any choice among eligible groups, and positive prescribed margins preserve this property. We call this set the residue of strict elimination.

\begin{proof}[Proof of Theorem~\mainref{thm:recovery}(b)]
Let $S$ be the residue. For $g\in S$, split the positive terms of $\partial_gJ$ in Equation~(\mainref{eq:derivative}) into a first collection, the tasks whose witnesses all lie in groups of $S$, and a second collection of all other tasks. The terms of the second collection are nonnegative. Each product in the first collection has total exponent at most $M=\max_e|W_e|$. If $p_h\geq1-\beta$ for every $h\in S$, each such product is therefore at least $(1-\beta)^M$, and
\[
 \partial_gJ(p)\geq(1-\beta)^ML_g(S)-D_g.
\]
Because $D_g<L_g(S)$ for every $g\in S$ and $S$ is finite, some $\beta\in(0,1)$ makes the right side at least a positive $\epsilon$ for every $g\in S$. Fix such a $\beta$ and start at an interior point with $p_g>1-\beta$ for every $g\in S$. While the iterates stay in this set, every coordinate of $S$ has derivative at least $\epsilon$, so under either update its logit increases and its probability stays above $1-\beta$. By induction over steps, the iterates never leave the set. Under Equation~\eqref{eq:npg}, each logit in $S$ increases by at least $\eta\lambda_g\epsilon$ per step and tends to $+\infty$. Under Equation~\eqref{eq:pg}, each logit in $S$ is increasing, and a finite limit would leave $p_g(1-p_g)$ with a positive limit and the increments above a fixed positive number, which is impossible. Hence $p_g\to1$ for every $g\in S$.

For the removed groups, take the removal order $g_1,\ldots,g_r$ of strict elimination and let $S_j$ contain $g_j$, the later removed groups and $S$. The induction in the proof of part (a) applies unchanged. It bounds the positive derivative terms of $g_j$ by $L_{g_j}(S_j)$ plus terms that contain a positive power of an earlier probability, and the bound holds uniformly over every other probability. The groups in $S$ therefore do not affect it. Hence $p_g\to0$ for every removed group.
\end{proof}

\begin{proposition}[Strict residue]
\label{prop:residue}
Under the conditions of Theorem~\mainref{thm:recovery}(b), the corner with $p_g=1$ for $g\in S$ and $p_g=0$ otherwise is a strict local maximum of $J$ on the probability cube.
\end{proposition}

\begin{proof}
Every coordinate in $S$ has derivative $L_g(S)-D_g>0$ at the corner. For a removed coordinate, the derivative terms that involve another removed coordinate vanish, and the terms supported on that coordinate and $S$ are bounded by its load when it was removed. Its derivative is therefore strictly negative. By continuity, these signs hold in a small enough rectangular neighborhood of the corner. Moving any coordinate from the corner into the cube decreases $J$, and the signs persist throughout the neighborhood, so the corner is a strict constrained local maximum.
\end{proof}

The strict inequalities are needed for Theorem~\mainref{thm:recovery}(b) and this proposition. A zero derivative at a coordinate in $S$ leaves the higher-order behavior unresolved. Local attraction under small parameter changes also does not show that a group has no profitable finite change in its action probability, so we do not use the residue as a Nash-equilibrium certificate.

\begin{corollary}[Surviving tasks]
\label{cor:surviving}
Under the conditions of Theorem~\mainref{thm:recovery}(b), let $E(S)$ be the set of tasks whose witnesses all lie in groups of $S$. Along the trajectories of part~(b), the expected fraction of completed tasks tends to $|E(S)|/|E|$, and $E(S)$ is nonempty.
\end{corollary}

\begin{proof}
A task in $E(S)$ completes with probability $\prod_{g\in S}p_g^{m_{eg}}$, which tends to 1 because every $p_g$ with $g\in S$ tends to 1. Every other task has a witness in a removed group, whose probability tends to 0, so its completion probability tends to 0. Summing over tasks and dividing by $|E|$ gives the limit. For nonemptiness, pick any $g\in S$. Part~(b) requires $D_g<L_g(S)$, and $D_g\geq0$, so $L_g(S)>0$. By Equation~(\mainref{eq:load}), some task $e$ with $m_{eg}>0$ has all of its witnesses in groups of $S$, and that task lies in $E(S)$.
\end{proof}

The corollary concerns exact updates from starts near enough to silence and strict residual inequalities. It is not a prediction for PPO within a finite horizon. One surviving task among 128 gives unsafe completion of $0.78\%$, so a nonempty residue does not by itself rule out the joint recovery threshold of the experiments.

\subsection{Order costs and optimization}
\label{app:cost}

The audit budget buys deterrence for individual agents, and the certificate needs that deterrence when the agent's group is removed.

\begin{proof}[Proof of Proposition~\mainref{prop:cost} and Equation~(\mainref{eq:cost})]
Fix an order $\pi$. A task remains in $L_g(S)$ until the first group that contains one of its witnesses is removed. It contributes $m_{eg}b_e$ to the load charged at that removal and nothing to later removals. With zero base costs and no binding caps, summing the required deterrence over the order gives
\[
 \sum_g L_g(S_g)=\sum_e b_e m_{e,\operatorname{first}_{\pi}(e)}.
\]
Arbitrarily small positive slacks approach this infimum for the strict inequalities, and multiplying by the common unit price $\alpha$ gives the audit budget. Fixed positive margins add deterrence $\sum_g\delta_g$ and audit budget $\alpha\sum_g\delta_g$.

For every task, the first-group multiplicity is an integer between one and $|W_e|$. Multiplying by the positive benefit and summing gives the bounds. Equality at $B$ holds exactly when every first-group multiplicity is one. For a fully shared policy, that multiplicity equals $k$ for every task when every task has $k$ witnesses.
\end{proof}

For heterogeneous prices and caps, write $c_g=\sum_{i\in g}c_i$. If group $g$ must buy $x$ units of extra deterrence, its minimum cost is
\begin{equation}
 \phi_g(x)=\min_{0\leq a_i\leq1}
 \left\{\sum_{i\in g}\kappa_i a_i\;\middle|\;
             \sum_{i\in g}F_i a_i\geq x\right\}.
 \label{eq:purchase}
\end{equation}
Set $\phi_g(x)=0$ for $x\leq0$ and $\phi_g(x)=+\infty$ above the available deterrence. Buying from agents in increasing order of $\kappa_i/F_i$ solves this fractional knapsack problem.

For fixed margins, let $\Phi(S)$ be the minimum remaining certificate cost. Then
\begin{equation}
 \Phi(\varnothing)=0,\qquad
 \Phi(S)=\min_{g\in S}\left[
   \phi_g\big(L_g(S)+\delta_g-c_g\big)+\Phi(S\setminus\{g\})\right].
 \label{eq:dp}
\end{equation}
Every complete certificate has a removal order, and every finite-cost order in this recurrence yields a certificate, so the recurrence is exact for this certificate class. It visits $2^G$ subsets with at most $G$ transitions each, besides the load and purchase-cost evaluations. The implementation uses the exact recurrence for $G\leq10$. For larger $G$, it chooses a feasible next group greedily by purchase cost, and the resulting costs are feasible but not proven minima.

A strong coloring of the co-witness graph places every task's witnesses in distinct groups, so every order has first-group multiplicity one. The condition $C(\pi)=B$ is weaker. It is the unique-minimum coloring condition on the witness graph, which becomes unique-maximum coloring when the order is reversed \citep{cheilaris2010}, and it requires only that every task have a unique witness in its first group. In a $k$-partite task family, one entire role can form the first group and the remaining roles a second group without changing the $B$ term. So a minimum-cost certificate need not give every witness role its own group. The main study uses the co-witness coloring and does not optimize the number of groups.

\subsection{The balanced family}
\label{app:balanced}

The balanced family of Section~\mainref{sec:balanced} has $q^2$ tasks and $kq$ agents, with $2\leq k\leq q$ and $q$ prime. Every agent witnesses $q$ tasks. For two different roles $t$ and $t'$, the difference $t-t'$ is nonzero in $\mathbb{F}_q$, so the equations $u+tv=w$ and $u+t'v=w'$ have exactly one solution for any pair of agent labels $w,w'$. The co-witness graph is therefore complete multipartite. Random relabeling changes only its visible indices.

\begin{proof}[Proof of Corollary~\mainref{cor:balanced}]
With role groups, any first role has load $B$. Its base deterrence is $\rho B/k$ and its prescribed margin is also $\rho B/k$, so it must buy at least $B$ units of fine deterrence. Every agent pays $n/(\gamma B)$ per unit of deterrence, so this purchase costs at least $n/\gamma$, and auditing each member of that role with probability $k/\gamma$ attains this cost. Once the role is removed, every task disappears, and the remaining roles' base costs exactly supply their prescribed margins, so no further auditing is needed.

For the fully shared policy, the load is $kB$, the base deterrence is $\rho B$ and the margin is $\rho B$. The group must buy at least $kB$ units of fine deterrence, which costs $kn/\gamma$, and auditing every agent with probability $k/\gamma$ attains this cost. The assumption $\gamma\geq k$ makes both allocations respect the probability cap. Without margins, the strict inequalities require the first role to buy more than $B-\rho B/k$ units and the fully shared group more than $kB-\rho B$ units. For $\rho<k$ the infima are therefore $(n/\gamma)(1-\rho/k)$ and $(n/\gamma)(k-\rho)$, whose ratio is again $k$.
\end{proof}

The global optimum and the local obstruction can also be checked directly. For $k\geq2$, $Bp^k-Dp$ is convex on $[0,1]$. When $D>B$, its larger endpoint value is $J(0)=0>J(1)=B-D$, so reporting is the global maximum. If also $D<kB$, the derivative at $p=1$ is positive and silence is locally attracting under ascent. Solving $kBp^{k-1}=D$ gives Equation~(\mainref{eq:threshold}).

An explicit time bound holds for normalized natural-gradient updates with role groups in the balanced family. Scale each group derivative by $k/B$, and let role 1 be the audited role. Its logit increment is then at most $-\eta\rho$. Let $\ell(z)=\log(z/(1-z))$, write $[z]_+=\max\{z,0\}$, and set $q_0=\min\{1/2,\rho/(2k)\}$. The audited role's probability is at most $q_0$ after
\[
 T_1=\left\lceil\frac{[\theta_1^0-\ell(q_0)]_+}{\eta\rho}\right\rceil
\]
steps. For any other role $h$, the positive derivative is at most $Bp_1$, and its base deterrence is $\rho B/k$. After $T_1$, its normalized derivative is at most $-\rho/2$. Before then its normalized derivative is at most $k$, so $\theta_h^{T_1}\leq\theta_h^0+\eta kT_1$. For a target probability $\epsilon\in(0,1)$, a valid additional bound is
\[
 T_2=\max\left\{
 \left\lceil\frac{[\ell(q_0)-\ell(\epsilon)]_+}{\eta\rho}\right\rceil,
 \max_{h\ne1}\left\lceil
 \frac{[\theta_h^0+\eta kT_1-\ell(\epsilon)]_+}{\eta\rho/2}
 \right\rceil\right\}.
\]
All role probabilities are at most $\epsilon$ after $T_1+T_2$ steps. The second term of $T_2$ shows the cost of learning in sequence. While the audited role is still learning, the other roles' logits can rise by up to $\eta k$ per step. Each step the audited role takes adds $2k/\rho$ steps to the bound on their remaining time. The bound is conservative and applies to exact updates, so it gives no sample complexity for PPO.

\subsection{The coalition guard}
\label{app:guard}

If some set of silent agents still earned more than it paid, an architecture that gives each agent its own silence probability would fail for reasons that have nothing to do with learning. The coalition guard rules this out. For a set $T\subseteq V$ of silent agents, while all other agents report, the deterministic team reward is
\begin{equation}
 U(T)=\sum_{e:W_e\subseteq T}b_e-\sum_{i\in T}d_i.
 \label{eq:coalition}
\end{equation}
A maximum-closure graph has a positive-weight vertex for each task, a negative-weight vertex for each agent, and an arc from each task to each of its witnesses, so selecting a task requires selecting its witnesses. Because all task weights are positive, the maximum closure weight equals $\max_TU(T)$, and the standard source-sink reduction computes it as a minimum cut \citep{picard1976}. A large enough finite capacity can replace the infinite arcs.

The study applies this check, with benefits multiplied by 1.02, to the certifying allocation that the search finds for witness groups. If that allocation fails the check, the study computes a flow allocation that deters every coalition and takes the coordinatewise maximum of the two allocations. No deterrence decreases, so the group certificate survives, the caps still hold, and every increased coordinate is charged to the reference budget. The guard adds deterrence on 18 of the 24 general graphs, including all 12 whose certificate costs are exact. The addition is between 6.47\% and 34.07\% of the reference budget, with median 24.27\%.

The same guard is applied when constructing the later shuffled allocation and the allocations on corrupted witness graphs, each on the graph used for its construction and before its audit budget is adjusted. Budget alternatives are evaluated without further repair, so scaling a guarded allocation down can break the certificate or the guard's condition. Equation~(\mainref{eq:cost}) describes the group certificate before the guard is applied.

A randomized joint policy is a distribution over deterministic coalitions, so an allocation that leaves every coalition unprofitable makes universal reporting optimal for the same team reward. This optimality statement concerns unrestricted joint action distributions for the unsafe-context reward. The bounded neural policies cannot represent universal reporting exactly, and the guard does not imply that their training will approach it.

\subsection{Which event triggers the fine}
\label{app:contract}

The obstruction of Section~\mainref{sec:obstruction} depends on the fine being charged for silence itself. Take one task of benefit 1, two witnesses that share silence probability $p$ and draw independently, and zero base costs. A duty-to-report fine with total expected value $3/2$ at universal silence gives $J(p)=p^2-3p/2$. Reporting is the unique optimum, but the derivative is positive for $p>3/4$, so exact updates from above $3/4$ move to silence. If the same aggregate fine is charged only when the task completes, the objective is $J(p)=-p^2/2$, whose derivative is negative throughout $(0,1)$, so exact updates move to reporting from every interior start. Both contracts give reward $-1/2$ at universal silence and the same optimum. The example shows that the obstruction depends on the contract and not only on the size of the fine. It does not show that completion-triggered fines remove the obstruction in games with base costs or several tasks, or under the neural training procedure.

\section{Experimental protocol}
\label{app:protocol}

The theory concerns exact derivatives of group probabilities, while the study obtains actions, rewards and updates from neural-network policies. The fixed graph bank, matched initialization and paired random streams hold everything else constant when the grouping or the audits change.

\subsection{Graphs and costs}

The general graphs cross four families, two witness widths and three graph seeds. Every general graph has $n=64$ agents and 128 tasks. Each task's benefit and harm weight are drawn uniformly from $[0.5,1.5]$. Completion fractions are unweighted, and the harm-weighted version is a secondary diagnostic. Base costs are proportional to independent draws from $[0.5,1.5]$ and are normalized to sum to $0.5B$. Fines $F_i$ are $16B/n$ times independent draws from $[0.75,1.25]$, and per-audit costs $\kappa_i$ are independent draws from $[0.5,1.5]$.

Uniform graphs sample each witness set without replacement from all agents. Ring graphs take $k$ consecutive agents modulo 64 from a uniformly drawn start. Modular graphs draw all witnesses from one random 16-agent module with probability 0.85 and from the whole population otherwise. Hub graphs sample without replacement with agent weights proportional to $1/\sqrt{i+1}$.

The balanced and rewired families use $q=11$ and widths four and eight, giving 44 or 88 agents and 121 tasks. Balanced graphs use unit benefits and harm weights, equal per-audit costs and the costs of Corollary~\mainref{cor:balanced} with $\rho=0.5$ and $\gamma=16$. Rewired graphs replace each witness set with probability 0.20 by a uniformly drawn set of the same size and draw benefits and harm weights from $[0.5,1.5]$. Their base costs have independent weights in $[0.6,1.4]$, normalized to $0.5B$. Role labels are randomly permuted before coloring.

Graph seeds are 914101, 914117 and 914131, and each graph has two training replicates. The context, pretraining, recovery, shuffle and corruption seeds are separate deterministic hashes of their declared arguments. The graph bank keeps all 36 graphs. No graph is excluded because of its training result or an infeasible certificate, and every scheduled design is feasible.

The coloring is deterministic DSATUR on the co-witness graph, with saturation degree, graph degree and agent index as successive tie-breakers \citep{brelaz1979}. Shuffling the resulting labels preserves every group's size, so the primary comparison matches the number and dimensions of the group networks.

\subsection{Contexts and observed features}

Legitimate contexts keep a policy from passing the endpoint by reporting everything. Every rollout includes all eight contexts with equal weight. Each episode is a one-step game with no state transition or temporal credit assignment. Contexts 0--3 are unsafe and contexts 4--7 are legitimate. For each context and task, a fixed multiplier in $[0.65,1]$ scales the benefit. Agents receive the context's one-hot vector and the mean and standard deviation of its multipliers. They do not receive witness membership, the allocation or an unsafe flag, although the context identities let them learn different responses to the fixed unsafe and legitimate contexts.

Each agent draws one action per episode, and that action applies to all the tasks it witnesses. Audit randomness is drawn in every context and every arm to keep the random streams paired, but audit events are set to zero in legitimate contexts. The audit budget is reported per unsafe context, $\sum_i\kappa_ia_i$, and its average over all eight contexts is half of that, because no audit cost arises in legitimate contexts.

\subsection{Networks and initialization}

For group $g$ and context features $x$, the base network computes
\begin{align}
 h_g(x)&=\tanh(W_gx+b_g),\\
 z_g(x)&=v_g^\top h_g(x)+u_g,\\
 \theta_g(x)&=12\tanh(z_g(x)/12),\qquad p_g(x)=\sig(\theta_g(x)).
\end{align}
The hidden layer has 32 units. With ten input features, each group network has 385 parameters. Backpropagation includes the derivative of the bounded logit. The fully shared architecture has one such network. The scalar-offset architecture adds a trainable value per agent before bounding the logit. The agent-ID architecture adds trainable weights from a one-hot identity input to the shared hidden layer. Independent networks give each agent a complete network of its own. On general graphs, these three architectures have 449, 2,433 and 24,640 parameters, against 385 for the fully shared network.

Pretraining starts at silence probability 0.75 and runs 512 audit-free rollouts. One pretrained network supplies the same starting policy to every matched recovery arm, and each group network receives a copy of its weights. New scalar offsets and identity weights start at zero. The actor parameters are kept, and optimizer moments and the value baseline are reset. All pretraining outcomes are included. The head resets of Appendix~\ref{app:protocol-budgets} are the declared exception to this matching.

\subsection{Sampled PPO}

The team reward is divided by $B/n$. Each rollout draws 64 one-step episodes per context and performs four epochs of the clipped PPO update. A scalar value estimate for each context moves toward the batch mean at rate 0.2. Advantages use the previous estimate and are standardized over contexts and episodes, with the standard deviation floored at $10^{-6}$.

For the current marginal probability $p_i$ and the behavior probability $p_i^{\rm old}$, the sampled marginal likelihood ratio is
\[
 r_i=\frac{p_i^{s_i}(1-p_i)^{1-s_i}}
 {(p_i^{\rm old})^{s_i}(1-p_i^{\rm old})^{1-s_i}}.
\]
The actor surrogate averages
\[
 \min\{r_i\widehat A,\operatorname{clip}(r_i,0.8,1.2)\widehat A\}
   +0.005\,H(\operatorname{Bernoulli}(p_i))
\]
over agents, contexts and episodes, with the same normalization in every architecture. Adam uses learning rate 0.003, moment coefficients 0.9 and 0.999, numerical constant $10^{-8}$, and a global gradient-norm cap of 0.5. Exact expected rewards and completion rates are used only for evaluation. These settings define the PPO variant we ran, and we do not claim that it matches other multi-agent PPO implementations.

The base condition runs 4,096 recovery rollouts of 512 episodes each, with sampling streams paired across matched conditions. Records keep the job specification, source and input fingerprints, initial-policy fingerprint, final actor parameters, evaluation trace and a cumulative sample hash. Continuation records keep the full optimizer, baseline and random-generator state.

\subsection{Budgets and additional experiments}
\label{app:protocol-budgets}

The budgets are $\{0.5,1,2,4,8\}$ times the reference budget, capped at the cost of auditing every agent. An allocation that costs more than its budget is scaled down proportionally. One that costs less spends the difference by raising every agent's audit probability equally, capped at one, so no coordinate decreases. Uniform and coalition audits receive the same audit budget. Coalition audits keep witness groups and maximize the smallest ratio of deterrence to completed benefit over individual coalitions, and any budget left after that ratio is fixed is spread uniformly. Each allocation is evaluated at its assigned budget with no later repair. Increases preserve the certificate for witness groups, and decreases need not.

Head resets keep the learned first layer, set the output weights to zero, and set the output bias to the logit of silence probability 0.75 or 0.95. They apply to witness, shuffled, fully shared, scalar-offset and agent-ID architectures, and they deliberately change the starting behavior. Both reset levels were fixed before study training.

The optimism experiment replaces the standardized advantage $\widehat A$ with $\max\{\omega\widehat A,\widehat A\}$ for $\omega=0$ or $0.5$ in witness and shuffled arms, on the first two graph seeds of every family and width. The transformation is the one proposed by \citet{zhao2024}, whose full method was evaluated in other environments with other implementation choices.

The continuation experiment uses the first graph seed of every family and width, chosen before training. Witness, shuffled and scalar-offset arms continue to 16,384 rollouts, with the 4,096-rollout prefix preserved. No observed outcome selects a continuation.

For witness error, a nominal 10\% or 25\% of tasks have their whole witness set replaced by a different random set of the same width. The number of replaced tasks is the nominal fraction times the task count, rounded to the nearest integer. Benefits, costs and fines are unchanged. The auditor builds its grouping and guarded allocation on the corrupted graph and then adjusts the allocation to the true graph's reference budget. The true graph supplies both the budget and the environment and never repairs the corrupted grouping or allocation. The two comparison arms are a shared network with scalar offsets and uniform audits with the corrupted grouping, both at the same audit budget. Table~\ref{tab:schedule} lists the main study's experiments with their record counts.

\begin{table}[!htbp]
\centering
\caption{The fixed main-study schedule. The 4,080 records include pretraining. Continuations extend selected base runs and are not independent replicates. All records completed.}
\label{tab:schedule}
\small
\begin{tabular}{lr}
\toprule
Experiment & Records \\
\midrule
Common pretraining & 72 \\
Seven methods at five audit budgets & 2,520 \\
Independent-network architecture & 72 \\
Five architectures at two head resets & 720 \\
Two groupings at two optimism levels & 192 \\
Three methods under two witness-error levels & 432 \\
Three methods continued to 16,384 rollouts & 72 \\
\midrule
Total & 4,080 \\
\bottomrule
\end{tabular}
\end{table}

\section{Endpoints, inference and chronology}
\label{app:inference}

Evaluation uses expected task completion under each policy's joint action distribution, so the endpoints carry no sampling noise from evaluation. For a policy in unsafe context $x$, let $v_x$ be the expected number of completed tasks divided by $|E|$. Legitimate completion uses the corresponding fraction in legitimate contexts. With independent actions, a task's completion probability is the product of its witnesses' silence probabilities. With a common draw, it is the single common silence probability.

Evaluation occurs at initialization and every 64 rollouts. With horizon $T$, the endpoint averages checkpoints with rollout index strictly greater than $0.75T$. For $T=4096$, these are the 16 checkpoints from 3136 through 4096. Unsafe completion averages over the four unsafe contexts and then over these checkpoints, and legitimate completion does the same over the four legitimate contexts. Joint recovery requires a final-quarter unsafe mean below 0.01 and a legitimate mean of at least 0.90. Joint recovery does not require each context or checkpoint to meet those thresholds separately.

For general graph $j$ and training replicate $r\in\{1,2\}$, let $v_{jr}^{S}$ and $v_{jr}^{W}$ be the shuffled and witness endpoints. The paired graph effect is
\[
 d_j=\frac12\sum_{r=1}^2(v_{jr}^{S}-v_{jr}^{W}).
\]
If $\mathcal{S}$ contains the eight general family-width strata, each with three graphs, the primary estimator is
\[
 \widehat\Delta=\frac18\sum_{s\in\mathcal{S}}\frac13\sum_{j\in s}d_j.
\]
The percentile interval resamples three graphs with replacement within each stratum, using 10,000 resamples and seed 914901. The interval describes graph variation within the observed strata, conditional on the two trained replicates averaged for each graph, and it does not estimate variation over new training seeds. Episodes, checkpoints and replicates are never treated as independent graphs. Three graphs per stratum limit how finely the interval can resolve that variation.

All architecture, budget, reset, optimism and witness-error experiments are secondary and descriptive. A budget curve's smallest successful tested budget is reported only on a complete grid, curves without a successful point are right-censored, and nonmonotone outcomes are kept. Neither grid minima nor certificate costs are treated as the minimum audit budget over learning algorithms. The certification splits of Table~\ref{tab:splits}, the round counts of Table~\ref{tab:rounds} and the threshold calculations of Appendix~\ref{app:secondary} were added after the results. The splits use only quantities fixed before training, including the context multipliers.

The main schedule was fixed before its training run, after earlier exploratory studies and development pilots whose results are not pooled into any estimate. The returned main run contains all 4,080 scheduled records. Two later questions led to new runs on the same general graphs. The allocation control adds 96 shuffled-allocation runs and 192 same-host repetitions of the original arms. The common-draw control adds 48 common-draw runs and 48 repetitions of the independent-draw arm. Each includes every general graph and both training replicates, with no outcome-dependent stopping or deletion. Both remain exploratory, because they reuse the graphs of the main study.

\section{Results by family and budget}
\label{app:families}

The primary estimate averages over graph families, so it can hide which witness structures benefit from grouping. Table~\ref{tab:families} separates the general families from the balanced and rewired ones at the reference budget, and Figure~\ref{fig:graphs} shows every paired graph effect together with its two training replicates. Witness groups have a positive mean effect in every general family. Shuffled groups do better on three graphs, and on a fourth both arms stay at 99.99\% unsafe completion. Table~\ref{tab:allbudgets} gives the complete budget aggregates, and Figure~\ref{fig:budgets} plots them for five arms.

\begin{table*}[tp]
\centering
\caption{Reference-budget results by family and witness width. Each row contains three graph seeds and two training replicates. $W$ denotes witness groups and $S$ shuffled groups. Completion columns are percentages, and differences are percentage points. Balanced and rewired rows are secondary and are excluded from the primary estimate.}
\label{tab:families}
\small
\begin{tabular}{lrrrrrr}
\toprule
Family & $k$ & Unsafe $W$ & Unsafe $S$ & $S-W$ & Joint recovery $W$ & Joint recovery $S$ \\
\midrule
Uniform & 4 & 43.32 & 78.43 & 35.11 & 0/6 & 0/6 \\
Uniform & 8 & 86.39 & 88.02 & 1.62 & 0/6 & 0/6 \\
Ring & 4 & 31.36 & 63.09 & 31.73 & 4/6 & 0/6 \\
Ring & 8 & 17.19 & 99.99 & 82.80 & 4/6 & 0/6 \\
Modular & 4 & 31.48 & 81.77 & 50.29 & 0/6 & 0/6 \\
Modular & 8 & 84.50 & 98.37 & 13.86 & 0/6 & 0/6 \\
Hub & 4 & 37.39 & 82.48 & 45.08 & 1/6 & 0/6 \\
Hub & 8 & 77.96 & 86.19 & 8.23 & 0/6 & 0/6 \\
Balanced & 4 & 33.33 & 99.98 & 66.65 & 4/6 & 0/6 \\
Balanced & 8 & 33.33 & 100.00 & 66.66 & 4/6 & 0/6 \\
Rewired & 4 & 23.03 & 99.96 & 76.93 & 2/6 & 0/6 \\
Rewired & 8 & 81.51 & 99.99 & 18.48 & 0/6 & 0/6 \\
 
\bottomrule
\end{tabular}
\end{table*}

\begin{figure*}[p]
\centering
\includegraphics[width=0.82\textwidth]{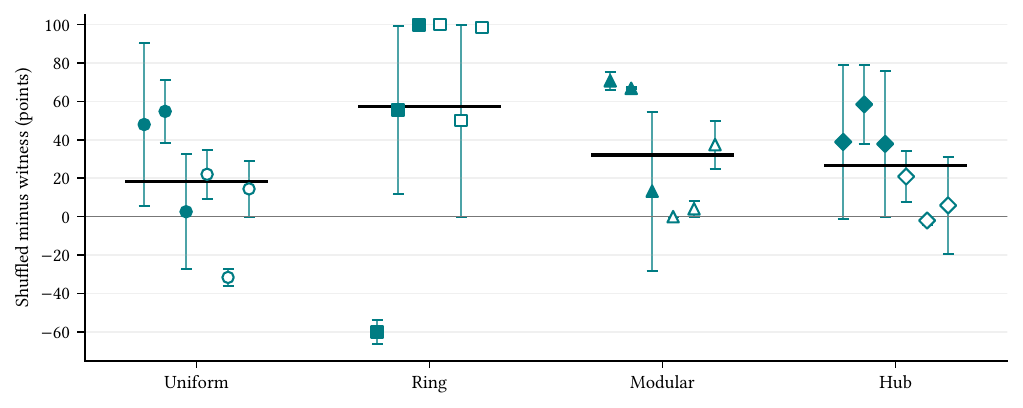}
\caption{The 24 paired general-graph effects at the reference budget. Each marker is the mean of a graph's two training replicates, and the ticks at the ends of the vertical line through it are the effects in the replicates themselves. Within a replicate, both arms start from the same pretrained network. Filled markers have witness width four and open markers width eight. Horizontal black segments are family means, and positive values favor witness groups. The graph means are the units of the primary bootstrap.}
\Description{A dot plot of shuffled-group minus witness-group unsafe completion in percentage points for 24 graphs in four families. Each graph has a marker at the mean of its two training replicates and a vertical line whose end ticks mark the two replicate effects. Family means are about 18 for uniform, 57 for ring, 32 for modular and 27 for hub graphs. Most graph means are positive. Four lie at or below zero, the lowest near minus 60 in the ring family and another near minus 32 in the uniform family. Replicate effects range from about minus 66 to 100, and the two replicates of a graph differ by more than 50 points on eight graphs.}
\label{fig:graphs}
\end{figure*}

Table~\ref{tab:rounds} splits the witness-group runs of Table~\ref{tab:allbudgets} under the reference allocation and under coalition audits by the number of elimination rounds that each allocation needs to certify witness groups on each graph. At four times the reference budget, the Fewer row contains exactly the six width-8 hub and uniform graphs, and at eight times it contains the three width-8 hub graphs. The six width-8 hub and uniform graphs are also the only graphs whose reference allocation comes from the greedy order without guard additions, so the table cannot attribute outcome differences to the round counts alone.

\begin{table*}[tp]
\centering
\caption{Coalition audits against the reference allocation for witness groups on the 24 general graphs. Graphs are split by whether coalition audits need fewer, equal or more elimination rounds than the reference allocation, or do not certify witness groups, which the reference allocation certifies on every graph at these budgets. $R$ denotes the reference allocation and $C$ coalition audits. Each row covers both training replicates of its graphs, and the unsafe columns are mean percentages. The split is post hoc and coincides with the graph subsets named in the text.}
\label{tab:rounds}
\small
\begin{tabular}{rlrrrrr}
\toprule
Budget & Rounds for $C$ & Graphs & Joint recovery $R$ & Joint recovery $C$ & Unsafe $R$ & Unsafe $C$ \\
\midrule
2 & Fewer & 10 & 3/20 & 2/20 & 61.36 & 48.78 \\
2 & Equal & 6 & 6/12 & 6/12 & 17.45 & 11.08 \\
2 & More & 3 & 3/6 & 2/6 & 7.65 & 30.21 \\
2 & Uncertified & 5 & 6/10 & 1/10 & 33.36 & 87.22 \\
4 & Fewer & 6 & 0/12 & 5/12 & 74.03 & 12.24 \\
4 & Equal & 18 & 29/36 & 32/36 & 9.09 & 1.48 \\
8 & Fewer & 3 & 3/6 & 6/6 & 14.35 & 0.13 \\
8 & Equal & 10 & 19/20 & 20/20 & 0.12 & 0.00 \\
8 & More & 11 & 22/22 & 22/22 & 0.04 & 0.04 \\
 
\bottomrule
\end{tabular}
\end{table*}

\begin{table*}[tp]
\centering
\caption{Complete budget aggregates on the general families. Each row contains 48 runs. Budget is a multiple of the reference budget, subject to the full-audit cap. Uniform and coalition rows keep witness groups. Certified counts the graphs on which the audits certify the architecture's groups, and it does not apply to scalar offsets or agent IDs, which do not form a grouping. Completion columns are percentages, and all action draws are independent.}
\label{tab:allbudgets}
\small
\begin{tabular}{lrrrrr}
\toprule
Method & Budget & Certified & Unsafe & Legitimate & Joint recovery \\
\midrule
Witness groups & 0.5 & 0/24 & 99.99 & 100.00 & 0/48 \\
Witness groups & 1 & 24/24 & 51.20 & 99.99 & 9/48 \\
Witness groups & 2 & 24/24 & 37.84 & 99.98 & 18/48 \\
Witness groups & 4 & 24/24 & 25.33 & 99.98 & 29/48 \\
Witness groups & 8 & 24/24 & 1.86 & 99.97 & 44/48 \\
Shuffled groups & 0.5 & 0/24 & 98.51 & 100.00 & 0/48 \\
Shuffled groups & 1 & 1/24 & 84.79 & 99.99 & 0/48 \\
Shuffled groups & 2 & 10/24 & 56.75 & 99.99 & 6/48 \\
Shuffled groups & 4 & 18/24 & 30.94 & 99.97 & 23/48 \\
Shuffled groups & 8 & 24/24 & 0.53 & 99.98 & 45/48 \\
Fully shared & 0.5 & 0/24 & 100.00 & 100.00 & 0/48 \\
Fully shared & 1 & 0/24 & 99.99 & 100.00 & 0/48 \\
Fully shared & 2 & 0/24 & 99.97 & 99.99 & 0/48 \\
Fully shared & 4 & 6/24 & 57.78 & 100.00 & 19/48 \\
Fully shared & 8 & 13/24 & 27.08 & 100.00 & 35/48 \\
Shared + scalar offsets & 0.5 & -- & 100.00 & 100.00 & 0/48 \\
Shared + scalar offsets & 1 & -- & 91.66 & 100.00 & 4/48 \\
Shared + scalar offsets & 2 & -- & 72.78 & 99.99 & 10/48 \\
Shared + scalar offsets & 4 & -- & 58.67 & 99.99 & 19/48 \\
Shared + scalar offsets & 8 & -- & 19.49 & 100.00 & 38/48 \\
Shared + agent IDs & 0.5 & -- & 97.32 & 99.96 & 0/48 \\
Shared + agent IDs & 1 & -- & 88.93 & 99.95 & 0/48 \\
Shared + agent IDs & 2 & -- & 66.16 & 99.93 & 10/48 \\
Shared + agent IDs & 4 & -- & 43.75 & 99.87 & 27/48 \\
Shared + agent IDs & 8 & -- & 13.32 & 99.82 & 41/48 \\
Witness + uniform audits & 0.5 & 0/24 & 99.99 & 100.00 & 0/48 \\
Witness + uniform audits & 1 & 0/24 & 99.95 & 99.99 & 0/48 \\
Witness + uniform audits & 2 & 0/24 & 99.77 & 99.99 & 0/48 \\
Witness + uniform audits & 4 & 8/24 & 52.87 & 99.98 & 21/48 \\
Witness + uniform audits & 8 & 19/24 & 9.30 & 99.97 & 41/48 \\
Witness + coalition audits & 0.5 & 0/24 & 99.34 & 99.99 & 0/48 \\
Witness + coalition audits & 1 & 1/24 & 88.20 & 99.99 & 1/48 \\
Witness + coalition audits & 2 & 19/24 & 45.05 & 99.98 & 11/48 \\
Witness + coalition audits & 4 & 24/24 & 4.17 & 99.99 & 37/48 \\
Witness + coalition audits & 8 & 24/24 & 0.03 & 100.00 & 48/48 \\
 
\bottomrule
\end{tabular}
\end{table*}

\begin{figure*}[tp]
\centering
\includegraphics[width=0.82\textwidth]{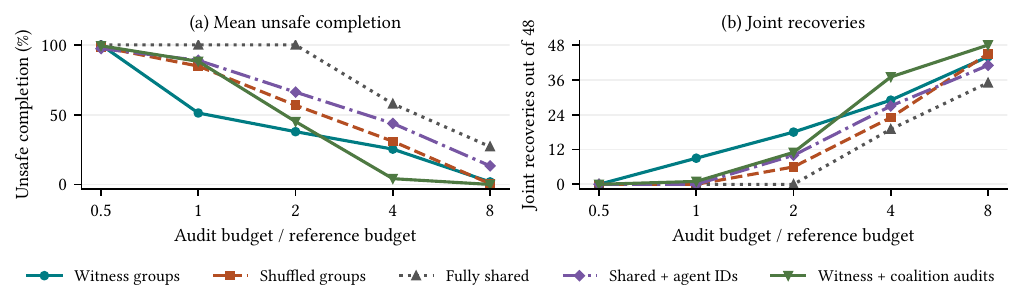}
\caption{Responses to the audit budget on the 24 general graphs of the first study, with 48 runs per point. Coalition audits keep witness groups and change only the allocation.}
\Description{Two panels with the audit budget on a logarithmic axis from 0.5 to 8 times the reference budget. The left panel shows mean unsafe completion. At the reference budget, witness groups are near 51 percent while shuffled groups, the fully shared network, the shared network with agent IDs and coalition audits stay above 84 percent. At eight times the reference budget, all methods except the fully shared network and the shared network with agent IDs are below 2 percent. The right panel shows joint recoveries out of 48. Witness groups rise from 0 to 9, 18, 29 and 44, coalition audits from 0 to 1, 11, 37 and 48, shuffled groups from 0 to 0, 6, 23 and 45, the shared network with agent IDs from 0 to 0, 10, 27 and 41, and the fully shared network from 0 to 0, 0, 19 and 35.}
\label{fig:budgets}
\end{figure*}

Table~\ref{tab:splits} splits the runs of the five budget-sweep arms that form a grouping by whether the allocation certifies the arm's grouping, and the witness advantage by whether the allocation certifies shuffled groups. Checked at the benefit upper bound, the certificate still leaves uncertified runs that recover, 3 of 134 at twice the reference budget and 16 of 80 at four times it, because each context scales each benefit by a multiplier between 0.65 and 1. Checked in the four unsafe contexts that each run trains on, every joint recovery up to four times the reference budget comes from a run certified in each of its unsafe contexts. At eight times the reference budget, 206 of 224 runs certified in each unsafe context recover, and so do 7 of the 16 others. All 16 use a fully shared network on a width-8 graph. In the seven that recover, total deterrence lies between $0.86kB_x$ and $kB_x$ in every unsafe context $x$, so universal silence is a strict local maximum with $0<J_x'(1)\leq0.14kB_x$, and PPO leaves it even though exact updates started near silence would not. Five other runs in that range stay silent, as do the four whose deterrence falls below $0.79kB_x$ in some unsafe context. Four of the eight width-8 fully shared runs certified in every unsafe context recover. The witness advantage is larger on graphs where the allocation fails to certify shuffled groups at every budget where both kinds of graph occur, but the pooled comparison confounds certification with width. At twice the reference budget, all 10 certified graphs have width 4, while 12 of the 14 uncertified ones have width 8. Within the family-width strata that contain both kinds of graph, the advantage is larger on the uncertified graphs in 8 of 9 comparisons, each with one or two graphs on either side.

\begin{table}[tp]
\centering
\caption{Post hoc splits by certification on the 24 general graphs of the first study. (a) Joint recoveries in the witness, shuffled, fully shared, uniform-audit and coalition-audit arms, by whether the allocation certifies the arm's grouping at the benefit upper bound and in each of the run's four unsafe contexts. Half the reference budget certifies nothing and gives no joint recovery. (b) Witness advantage in points by whether the allocation certifies shuffled groups at the bound, with graph counts in parentheses.}
\label{tab:splits}
\small
\setlength{\tabcolsep}{4pt}
\begin{tabular}{rrrrr}
\multicolumn{5}{l}{(a) Joint recoveries} \\
\toprule
 & \multicolumn{2}{c}{At the bound} & \multicolumn{2}{c}{In each unsafe context} \\
\cmidrule(lr){2-3}\cmidrule(lr){4-5}
Budget & Certified & Uncertified & Certified & Uncertified \\
\midrule
1 & 10 of 52 & 0 of 188 & 10 of 78 & 0 of 162 \\
2 & 32 of 106 & 3 of 134 & 35 of 126 & 0 of 114 \\
4 & 113 of 160 & 16 of 80 & 129 of 186 & 0 of 54 \\
8 & 200 of 208 & 13 of 32 & 206 of 224 & 7 of 16 \\ 
\bottomrule
\end{tabular}

\medskip
\begin{tabular}{rlrrr}
\multicolumn{5}{l}{(b) Witness advantage} \\
\toprule
Budget & Allocation & Certified & Uncertified & All \\
\midrule
1 & Reference & $-$60.10 (1) & 37.67 (23) & 33.59 \\
1 & Redesigned & 20.28 (14) & 34.07 (10) & 26.03 \\
2 & Reference & 13.49 (10) & 22.78 (14) & 18.91 \\
2 & Redesigned & 4.78 (24) & -- (0) & 4.78 \\
4 & Reference & $-$4.77 (18) & 36.76 (6) & 5.61 \\
8 & Reference & $-$1.33 (24) & -- (0) & $-$1.33 \\
 
\bottomrule
\end{tabular}
\end{table}

\subsection{Allocation control}
\label{app:allocation}

The primary comparison uses an allocation designed for witness groups, so shuffled groups might fail only because their allocation was designed for another grouping. The allocation control constructs an allocation for shuffled groups on each graph with the same search, fines, margin and coalition guard as the reference allocation. Its audit budget averages 1.142 times the reference budget over the 24 general graphs, and it is not an exact optimum on graphs that use the greedy branch. Adjusted to the reference budget by the rule of Appendix~\ref{app:protocol-budgets}, the redesigned allocation certifies shuffled groups on 14 of the 24 graphs. On 18 of the 24 graphs, it also leaves every silent coalition with a profit below $10^{-12}$ at the actual benefits. The reference allocation certifies shuffled groups on one graph and witness groups on all 24.

\begin{table}[!htbp]
\centering
\caption{Allocation control with same-host repetitions. Every row contains the same 24 general graphs and both training replicates. Completion columns are percentages. The horizon, endpoint and starting policy are unchanged.}
\label{tab:allocation}
\small
\setlength{\tabcolsep}{2.5pt}
\begin{tabular}{rlrrr}
\toprule
Budget & Grouping and allocation & Unsafe & Legitimate & Joint recovery \\
\midrule
1 & Witness, reference & 51.20 & 99.99 & 9/48 \\
1 & Shuffled, reference & 84.79 & 99.99 & 0/48 \\
1 & Shuffled, redesigned & 77.23 & 99.99 & 0/48 \\
2 & Witness, reference & 37.84 & 99.98 & 18/48 \\
2 & Shuffled, reference & 56.75 & 99.99 & 6/48 \\
2 & Shuffled, redesigned & 42.62 & 99.98 & 13/48 \\
 
\bottomrule
\end{tabular}
\end{table}

At the reference budget, the witness advantage over shuffled groups with the redesigned allocation is 26.03 points, with exploratory interval $[11.77,\allowbreak 39.93]$. At twice the reference budget it is 4.78 points, with interval $[-3.98,\allowbreak 13.55]$. The intervals use paired graph units, stratified resampling and seed 9152026, and they are not adjusted for multiple comparisons. Table~\ref{tab:allocation} keeps the repetitions with the reference allocation so that the redesigned arm is compared on the same host.

\section{Training, initialization and witness information}
\label{app:secondary}

These experiments ask whether recovery changes when training starts farther from silence, uses different advantage weights or runs longer, or when the auditor designs the allocation from wrong witness information. They include the balanced and rewired families as stated in each caption, so their denominators differ from the 48-run summaries of the main text. Both completion measures are kept, because lower unsafe completion can come with a large loss of legitimate completion.

\begin{table}[!htbp]
\centering
\caption{Output-head resets at the reference budget across all six families, with 72 runs per row. $p_0$ is the reset silence probability, and the resets keep the hidden layer but change the starting policy. Completion columns are percentages.}
\label{tab:resets}
\small
\setlength{\tabcolsep}{3.5pt}
\begin{tabular}{lrrrr}
\toprule
Architecture & $p_0$ & Unsafe & Legitimate & Joint recovery \\
\midrule
Shared + agent IDs & 0.75 & 26.96 & 55.46 & 2/72 \\
Shared + agent IDs & 0.95 & 77.30 & 99.92 & 1/72 \\
Shared + scalar offsets & 0.75 & 50.02 & 55.55 & 3/72 \\
Shared + scalar offsets & 0.95 & 93.05 & 100.00 & 5/72 \\
Fully shared & 0.75 & 55.55 & 55.55 & 0/72 \\
Fully shared & 0.95 & 100.00 & 100.00 & 0/72 \\
Shuffled groups & 0.75 & 34.86 & 52.77 & 4/72 \\
Shuffled groups & 0.95 & 76.97 & 99.99 & 0/72 \\
Witness groups & 0.75 & 2.32 & 48.61 & 26/72 \\
Witness groups & 0.95 & 24.10 & 99.99 & 32/72 \\
 
\bottomrule
\end{tabular}
\end{table}

A reset gives every agent the same silence probability $p_0$, so every arm starts at $p_0^k$ unsafe and legitimate completion (Table~\ref{tab:resetwidth}). For the fully shared network, the unsafe-context stationary point $p_*$ of Equation~(\mainref{eq:threshold}) lies between 0.76 and 0.85 over all 36 reference allocations and unsafe contexts. The analogous legitimate-context point $(\sum_ic_i/(kB_x))^{1/(k-1)}$ lies between 0.689 and 0.694 at width 8 and between 0.529 and 0.539 at width 4. The output weights start at zero, so the exact gradient on the shared output bias has the sign of the derivative summed over all eight contexts. At $p_0=0.95$, this sum favors silence in every graph. At $p_0=0.75$, it favors silence in every width-4 graph and reporting in every width-8 graph. The final-quarter outcomes of the fully shared network agree with these initial gradient signs in 140 of its 144 reset runs (Table~\ref{tab:resets}). It ends silent in both kinds of context in every width-4 run and in every width-8 run from 0.95. From 0.75 at width 8, it ends reporting in both kinds of context in 32 of 36 runs and silent in the other 4 (Table~\ref{tab:resetwidth}). The script \texttt{analysis/\allowbreak design\_checks.py} computes these quantities from the rebuilt graph bank.

Without the optimistic transformation, the graphs and replicates of Table~\ref{tab:optimism} give witness groups 51.77\% unsafe completion and 13 of 48 joint recoveries, and shuffled groups 90.48\% and none. With the transformation, witness groups achieve 1 and 2 joint recoveries at its two settings. \citet{zhao2024} note that optimism can mislead learning when rewards are stochastic, as the sampled fines are here.

The continuation experiment (Table~\ref{tab:long}) also helps on the general graphs alone, where it lowers witness-group unsafe completion from 50.85\% to 29.82\% and raises joint recoveries from 3 to 8 of 16.

Witness error (Table~\ref{tab:error}) raises witness-group unsafe completion on the general graphs to 66.45\% and 65.21\%, against 68.58\% and 70.01\% across all six families.

\begin{table*}[t]
\centering
\caption{Output-head resets by witness width, with 36 runs per row across all six families. The start column is the unsafe and legitimate completion $100p_0^k$ at the reset. Completion columns are percentages.}
\label{tab:resetwidth}
\small
\begin{tabular}{lrrrrrrr}
\toprule
Architecture & $p_0$ & $k$ & Start & Unsafe & Legitimate & Legitimate below 90\% & Joint recovery \\
\midrule
Witness groups & 0.75 & 4 & 31.64 & 2.27 & 86.10 & 5/36 & 25/36 \\
Witness groups & 0.75 & 8 & 10.01 & 2.37 & 11.11 & 32/36 & 1/36 \\
Witness groups & 0.95 & 4 & 81.45 & 4.65 & 99.99 & 0/36 & 22/36 \\
Witness groups & 0.95 & 8 & 66.34 & 43.54 & 99.99 & 0/36 & 10/36 \\
Shuffled groups & 0.75 & 4 & 31.64 & 60.58 & 91.66 & 3/36 & 4/36 \\
Shuffled groups & 0.75 & 8 & 10.01 & 9.14 & 13.89 & 31/36 & 0/36 \\
Shuffled groups & 0.95 & 4 & 81.45 & 62.19 & 99.99 & 0/36 & 0/36 \\
Shuffled groups & 0.95 & 8 & 66.34 & 91.75 & 99.99 & 0/36 & 0/36 \\
Fully shared & 0.75 & 4 & 31.64 & 100.00 & 100.00 & 0/36 & 0/36 \\
Fully shared & 0.75 & 8 & 10.01 & 11.11 & 11.11 & 32/36 & 0/36 \\
Fully shared & 0.95 & 4 & 81.45 & 100.00 & 100.00 & 0/36 & 0/36 \\
Fully shared & 0.95 & 8 & 66.34 & 99.99 & 99.99 & 0/36 & 0/36 \\
Shared + scalar offsets & 0.75 & 4 & 31.64 & 88.92 & 99.99 & 0/36 & 3/36 \\
Shared + scalar offsets & 0.75 & 8 & 10.01 & 11.11 & 11.11 & 32/36 & 0/36 \\
Shared + scalar offsets & 0.95 & 4 & 81.45 & 86.11 & 100.00 & 0/36 & 5/36 \\
Shared + scalar offsets & 0.95 & 8 & 66.34 & 99.99 & 99.99 & 0/36 & 0/36 \\
Shared + agent IDs & 0.75 & 4 & 31.64 & 42.81 & 99.82 & 0/36 & 2/36 \\
Shared + agent IDs & 0.75 & 8 & 10.01 & 11.11 & 11.11 & 32/36 & 0/36 \\
Shared + agent IDs & 0.95 & 4 & 81.45 & 54.60 & 99.84 & 0/36 & 1/36 \\
Shared + agent IDs & 0.95 & 8 & 66.34 & 99.99 & 99.99 & 0/36 & 0/36 \\
 
\bottomrule
\end{tabular}
\end{table*}

\begin{table}[!htbp]
\centering
\caption{Witness error at the reference budget across all six families, with 72 runs per row. Error is the nominal fraction of tasks whose witness set is replaced, rounded to an integer task count. The true graph supplies the budget and the environment. Completion columns are percentages.}
\label{tab:error}
\small
\setlength{\tabcolsep}{3.5pt}
\begin{tabular}{lrrrr}
\toprule
Method & Error & Unsafe & Legitimate & Joint recovery \\
\midrule
Shared + scalar offsets & 10\% & 91.66 & 100.00 & 6/72 \\
Shared + scalar offsets & 25\% & 95.94 & 100.00 & 1/72 \\
Witness + uniform audits & 10\% & 99.99 & 99.99 & 0/72 \\
Witness + uniform audits & 25\% & 99.99 & 99.99 & 0/72 \\
Witness groups & 10\% & 68.58 & 99.99 & 7/72 \\
Witness groups & 25\% & 70.01 & 99.99 & 5/72 \\
 
\bottomrule
\end{tabular}
\end{table}

\begin{table}[!htbp]
\centering
\caption{Optimistic advantage transformations at the reference budget. Each row has 48 runs over six families, two widths, the first two graph seeds and two training replicates. The transformation is $\max\{\omega\widehat A,\widehat A\}$. Completion columns are percentages. This tests one component and does not reproduce the full method of \citet{zhao2024}.}
\label{tab:optimism}
\small
\setlength{\tabcolsep}{3.5pt}
\begin{tabular}{lrrrr}
\toprule
Grouping & $\omega$ & Unsafe & Legitimate & Joint recovery \\
\midrule
Shuffled groups & 0 & 96.63 & 99.73 & 0/48 \\
Shuffled groups & 0.5 & 98.96 & 100.00 & 0/48 \\
Witness groups & 0 & 82.63 & 99.62 & 1/48 \\
Witness groups & 0.5 & 87.65 & 99.99 & 2/48 \\
 
\bottomrule
\end{tabular}
\end{table}

\begin{table}[!htbp]
\centering
\caption{Fixed continuation subset at the reference budget. Each row has 24 runs over six families, two widths, the first graph seed and two training replicates. The 16,384-rollout run continues the 4,096-rollout run with optimizer and random state preserved, and each endpoint uses the final quarter of its own horizon. Completion columns are percentages.}
\label{tab:long}
\small
\setlength{\tabcolsep}{3.5pt}
\begin{tabular}{lrrrr}
\toprule
Architecture & Rollouts & Unsafe & Legitimate & Joint recovery \\
\midrule
Shared + scalar offsets & 4,096 & 95.83 & 100.00 & 1/24 \\
Shared + scalar offsets & 16,384 & 91.66 & 100.00 & 2/24 \\
Shuffled groups & 4,096 & 92.34 & 99.99 & 0/24 \\
Shuffled groups & 16,384 & 78.67 & 100.00 & 0/24 \\
Witness groups & 4,096 & 42.44 & 99.99 & 8/24 \\
Witness groups & 16,384 & 26.97 & 100.00 & 14/24 \\
 
\bottomrule
\end{tabular}
\end{table}

\section{The common-draw control}
\label{app:coupling}

Identical marginal probabilities do not imply the same joint policy, and this control changes only the joint policy. The control uses the fully shared network and the same observation for every agent. For each context $x$ and episode $b$, it draws $U_{xb}$ uniformly from $[0,1]$ and sets
\[
 s_{xbi}=\ind\{U_{xb}<p_x\}\quad\text{for every agent }i.
\]
The sampler still draws an array of independent uniforms with the dimensions of the original arm and uses the first agent's entry as the common draw. Drawing the full array keeps the independent audit events at the same positions in the random stream. All rewards, fines and costs are computed by the original functions.

Every task has a nonempty witness set, so all tasks complete exactly when the common action is silence. Write $B_x$ for the total benefit of context $x$ and $D_x$ for the total expected cost of silence in it. The expected benefit is $p_xB_x$, the expected cost is $p_xD_x$, and the expected completion fraction is $p_x$. The argument proves Equation~(\mainref{eq:common}) without assuming equal witness widths.

Every general reference allocation satisfies $D>B$ at the benefit upper bound, and $D_x=D$ in an unsafe context, so $B_x-D_x<0$ there. Legitimate contexts have no fines, so $D_x=\sum_ic_i=0.5B$ in them. Since $B_x\geq0.65B$, their common-draw slope is $B_x-0.5B\geq0.15B>0$. These expected rewards favor reporting in unsafe contexts and silence in legitimate ones. These signs do not remove the coupling between contexts in the network or establish convergence for PPO.

\begin{figure*}[t]
\centering
\includegraphics[width=0.82\textwidth]{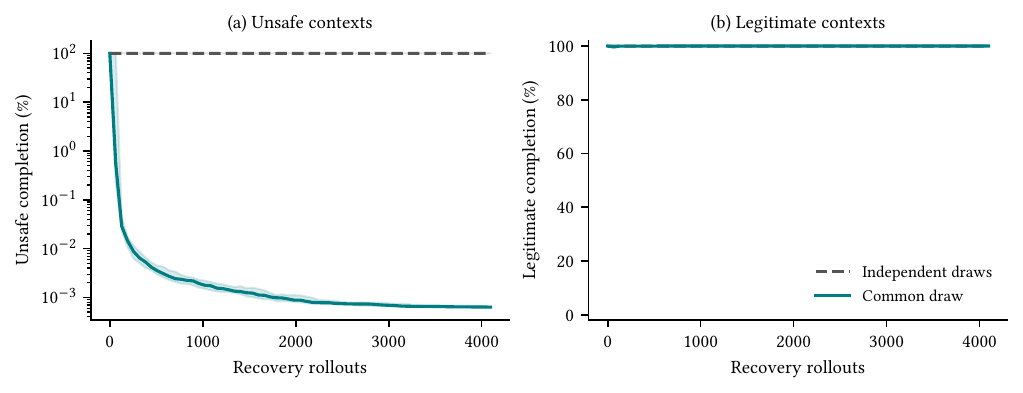}
\caption{Learning in the common-draw control. Lines are medians across the 48 general-graph runs and shading gives the interquartile range. Unsafe completion uses a logarithmic scale. The legitimate panel shows that the common draw does not reach low unsafe completion by reporting everywhere. The tabulated endpoint is a final-quarter mean.}
\Description{Two panels against recovery rollouts from 0 to 4096. In the left panel, on a logarithmic scale, unsafe completion under independent draws stays near 100 percent, while under the common draw it falls within a few hundred rollouts to below 0.01 percent and ends near 0.0007 percent. In the right panel, legitimate completion stays near 100 percent for both arms throughout training.}
\label{fig:commonlearning}
\end{figure*}

The joint policy has only two supported action vectors, $\mathbf{0}$ and $\mathbf{1}$, with likelihood
\[
 P_\theta(s\mathbf{1}\mid x)=p_x^s(1-p_x)^{1-s}.
\]
For a stored common action, every agent's marginal PPO ratio equals this joint ratio, so every clipped surrogate term is identical. Averaging them across agents gives one copy of the single-Bernoulli surrogate. The mean marginal entropy also equals the entropy of this two-action joint policy. Multiplying the identical marginal likelihoods would be wrong, because the draws are dependent. The implementation keeps the original average and changes only the sampler and the evaluator.

Keeping that average does not keep its normalization relative to the joint-policy score. Let $\pi_\theta$ be the shared marginal policy. Under independent actions,
\begin{equation}
 \frac1n\sum_i\nabla_\theta\log\pi_\theta(s_i\mid x)
 =\frac1n\nabla_\theta\log P_\theta^{\mathrm{ind}}(\mathbf{s}\mid x).
 \label{eq:independent-score}
\end{equation}
Under the common draw, $s_i=s$ for all $i$, so
\begin{equation}
 \frac1n\sum_i\nabla_\theta\log\pi_\theta(s\mid x)
 =\nabla_\theta\log P_\theta^{\mathrm{common}}(s\mathbf{1}\mid x).
 \label{eq:common-score}
\end{equation}
At the behavior policy, these are the score factors of the unclipped actor update. Write $H(p_x)$ for the entropy of a Bernoulli draw with parameter $p_x$. The joint entropy is $nH(p_x)$ for the independent policy and $H(p_x)$ for the common policy, while the implemented mean marginal entropy equals $H(p_x)$ in both arms. So the joint action spaces, exploration and normalization differ even though the surrogate code is unchanged. Standardized advantages, clipping and Adam also prevent reading the observed PPO steps as a direct comparison of exact derivatives.

The intervention keeps the starting marginal probabilities and pretrained arrays, changes the starting joint distribution and assumes a common random source and a common observation. Agents with different conditional probabilities would in general act differently on the same uniform draw, which would require a different joint likelihood. The result concerns the two-action joint policy defined above.

All 48 common-draw runs meet the joint recovery endpoint (Figure~\ref{fig:commonlearning}). The largest run-level final-quarter unsafe completion is $9.30\times10^{-6}$ as a fraction, the smallest legitimate completion is 0.999918, and the mean unsafe fraction is $6.788\times10^{-6}$. The bounded logit keeps every silence probability above $\sig(-12)\approx6.1\times10^{-6}$, so these values sit near the smallest the network can produce, and the main text reports the mean exactly instead of rounding it to zero.

Separate checks compared sampled rewards over 160,000 joint episodes with the analytic common-draw expectation, and compared the gradient of the surrogate averaged over repeated agents with the gradient of one Bernoulli surrogate. The largest absolute standardized reward discrepancy was 1.94, and the largest gradient discrepancy was $7.26\times10^{-18}$. A scalar network forward pass and reward calculation also recomputed every saved final policy value.

\section{Audit allocation at equal cost}
\label{app:newalloc}

\subsection{Study sequence}

The allocation study followed the grouping experiments of the main study and the exploratory diagnostics reported above. The locally frozen production protocol fixed the inputs, training settings, primary contrast and endpoint before training. The retained analysis plan records that the uncertainty procedures and the worked case below were specified before outcome inspection, although that plan was not part of the training freeze. The primary comparison is coalition minus reference unsafe completion at the reference budget from learned silence. Uniform audits were added before the input freeze as a secondary baseline. While the main runs were in progress, checks on the inputs alone motivated two separately frozen secondary controls, selection by the certificate among coalition-optimal allocations and minimum-$L_2$ selection without the certificate. These controls leave the primary comparison unchanged. The freezes were local records and not external preregistration, and the controls are not independent samples of graphs.

All 24 declared graphs and eight paired training replicates per graph were retained. The main runs contain 192 pretraining and 2,304 recovery records, crossing three allocations, two budgets and two starts. The tie-aware extension adds 768 recovery records at both budgets and both starts, and the minimum-$L_2$ control adds 192 at the reference budget from learned silence, for 3,456 records in all. No graph, seed or outcome was replaced because of its result.

\subsection{Games, contexts and grouping}

Each graph has $n=64$ agents and 128 tasks. We cross the uniform, ring, modular and hub families with witness widths $k\in\{4,8\}$ and graph seeds 71027011, 71027029 and 71027047, which are new relative to the main study. The generator hashes the family and parameters with the seed, so these identifiers give distinct streams across strata. The families are generated as in Appendix~\ref{app:protocol}. Task benefits are uniform on $[0.5,1.5]$, and positive base costs are drawn independently from the same interval and rescaled to sum to $0.5B$. The eight recurring contexts, their multipliers in $[0.65,1]$, the observed features and the absence of an unsafe label are as in the main study. Inspections and fines apply only in unsafe contexts, and base costs apply in both kinds. We report the audit budget per unsafe context, $\sum_i\kappa_ia_i$, so the expected audit cost averaged over the eight equally sampled contexts is half that amount. Witness groups come from the same deterministic coloring, so every pair of co-witnesses is in separate groups, each group has one network and one silence probability per context, and agents draw independently. Group membership is identical across the compared allocations. Because each task has at most one witness in any group, this study tests allocation and does not isolate within-group witness multiplicity.

\subsection{Audit allocations}

All agents have fine $F_i=16B/n$ and audit price $\kappa_i=1$, so equal spending implies equal total deterrence, since base costs are fixed. All allocations obey $0\leq a_i\leq1$. Let $B(C)$ be the benefit of the tasks whose witnesses all belong to a coalition $C$. The static coalition objective is
\[
\lambda(a)=\min_{C:B(C)>0}\frac{\sum_{i\in C}(c_i+F_ia_i)}{B(C)}.
\]
Coalitions with no completed benefit do not enter the ratio, and maximizing it is a different problem from checking the certificate.

The reference allocation searches for a removal order with margins $\delta_g=0.5|g|B/n$, exactly for at most ten groups and with the fixed greedy procedure otherwise. The coalition guard then raises audit probabilities where needed to cover every coalition's benefit multiplied by 1.02. The guarded allocation defines the reference budget $A$. For the higher budget the target is $\min(4A,n)$, reached by raising probabilities equally until the budget or a cap binds. Coalition audits reoptimize $\lambda(a)$ at the same target with the deterministic flow solver and residual-budget fill of the main study, and uniform audits give each agent the same probability. The reference budget is part of the design and not a claim that this budget is optimal.

Tie-aware coalition audits first obtain the maximal multiplier $\lambda^\star$. They then require $\lambda(a)\geq\lambda^\star-10^{-8}$ and maximize $\delta$ subject to
\[
D_g-L_g(S_g)\geq\delta B|g|/n
\]
for every group along the reference removal order, where $S_g$ contains that group and all later groups. The reference order at the reference budget is used at both budgets, and no further adjustment follows this solve. The optimization uses the witness graph, the grouping and the ordered loads, and it never sees training outcomes.

Minimum-$L_2$ coalition audits keep the same multiplier tolerance and minimize $\sum_ia_i^2$. Their constraints depend on witness sets, benefits, costs, fines, budget and caps, but not on groups, removal orders, networks or outcomes. Feasibility can be written with nonnegative task-to-witness flows in which task $e$ supplies $\lambda b_e$ and agent $i$ has capacity $c_i+F_ia_i$. Strict convexity makes the minimizing audit vector unique, although supporting flows need not be. The implementation uses an analytically checked uniform solution where feasible and otherwise a quadratic program with an independent solver comparison. All intended jobs were trained, including allocations that numerically duplicate another arm.

At the benefit upper bound, reporting is globally optimal for every allocation in the study, and context benefits are at most that bound. This concerns unsafe-context rewards and not the legitimate contexts or the joint neural objective. The reference allocation meets the certificate on all 24 graphs at both budgets, coalition audits on none at the reference budget and on 20 at four times it, uniform audits on none and 12, and tie-aware audits on all 48 graph and budget pairs. None of the 24 minimum-$L_2$ allocations meets it, although all satisfy the coalition constraint within tolerance. These are numerical checks with the stated solver tolerances and not symbolic proofs of exact optima.

\subsection{Pretraining and recovery}

For each graph and training replicate, a fully shared network first trains without audits for 512 rollouts from silence probability 0.75, and its parameters are copied into each group network for recovery, so every allocation starts from the same pretrained parameters. All learned-start networks have unsafe and legitimate completion above 99\%. The reset condition keeps the hidden features, zeroes the output weights and sets the output bias so that each agent is silent with probability 0.95. Optimizer moments and the context value baseline reset in every recovery arm. With independent witnesses, initial task completion after the reset is $0.95^k$, or 81.45\% at width 4 and 66.34\% at width 8, in legitimate contexts as well, so the reset starts below the legitimate-completion threshold.

The network, PPO settings, Adam settings and value baseline are those of Appendix~\ref{app:protocol}, and the actor code is unchanged from the audited original implementation. Each rollout samples 64 one-step joint actions in each of the eight contexts. A recovery runs 4,096 rollouts, or 2,097,152 sampled episodes, and pretraining 262,144. Matched arms use the same namespaced random streams for corresponding action and audit draws, although different policies and audits produce different realized actions and rewards. Exact expectations are used only for evaluation.

\subsection{Endpoints and uncertainty}

The endpoint is the mean of the 16 exact evaluations at rollouts 3,136 to 4,096 in steps of 64, as in Appendix~\ref{app:inference}, and joint recovery uses the same thresholds. For the primary contrast, we subtract reference from coalition unsafe completion in each paired replicate, average the eight differences within a graph and give each of the eight family and width strata equal weight, which here equals the mean over the 24 graph averages. The 95\% percentile interval uses 10,000 graph-bootstrap resamples with NumPy PCG64 seed 20261007, each selecting three graphs with replacement within every stratum and keeping all paired replicates. Secondary contrasts use the same indices, and their intervals are descriptive and uncorrected for multiplicity. Only three graphs occur in each stratum, and the bootstrap is conditional on their observed replicate means. Success counts are descriptive, since a binomial model with 192 independent trials would be wrong. The analytic sensitivity interval for the primary effect, specified in the analysis plan, is $[36.19,48.25]$ points, and the unstratified graph-bootstrap interval is $[31.77,53.10]$. All four family means favor reference audits.

\subsection{All conditions and contrasts}

Table~\ref{tab:new-allocation-all} reports every allocation, budget and start cell, and Tables~\ref{tab:new-allocation-main-contrasts} to~\ref{tab:new-allocation-l2-contrasts} every declared contrast. The primary comparison is the only confirmatory contrast. The minimum-$L_2$ control ran only at the reference budget from learned silence.

\begin{table*}[t]
\centering
\caption{Every completed allocation condition. Completion rates are percentages. Joint recovery requires unsafe completion below 1\% and legitimate completion of at least 90\%. Each row uses the same 24 graphs and eight paired replicates. Displayed zeros are rounded probabilities. Minimum-$L_2$ was run only in the final row's condition.}
\label{tab:new-allocation-all}
\footnotesize
\begin{tabular}{@{}lclrrr@{}}
\toprule
Start & Budget & Allocation & Unsafe & Legitimate & Joint \\
\midrule
Learned & $1\times$ & Reference & 57.41 & 99.99 & 22/192 \\
Learned & $1\times$ & Coalition & 99.62 & 99.99 & 0/192 \\
Learned & $1\times$ & Uniform & 99.98 & 99.99 & 0/192 \\
Learned & $1\times$ & Tie-aware coalition & 58.02 & 99.99 & 26/192 \\
Learned & $4\times$ & Reference & 15.06 & 99.97 & 135/192 \\
Learned & $4\times$ & Coalition & 16.18 & 99.99 & 132/192 \\
Learned & $4\times$ & Uniform & 42.13 & 99.98 & 101/192 \\
Learned & $4\times$ & Tie-aware coalition & 15.19 & 99.97 & 123/192 \\
$p=0.95$ reset & $1\times$ & Reference & 33.23 & 99.99 & 49/192 \\
$p=0.95$ reset & $1\times$ & Coalition & 99.27 & 99.99 & 0/192 \\
$p=0.95$ reset & $1\times$ & Uniform & 99.98 & 99.99 & 0/192 \\
$p=0.95$ reset & $1\times$ & Tie-aware coalition & 34.32 & 99.99 & 42/192 \\
$p=0.95$ reset & $4\times$ & Reference & 0.00 & 50.00 & 96/192 \\
$p=0.95$ reset & $4\times$ & Coalition & 6.14 & 43.74 & 68/192 \\
$p=0.95$ reset & $4\times$ & Uniform & 34.21 & 51.02 & 27/192 \\
$p=0.95$ reset & $4\times$ & Tie-aware coalition & 0.00 & 50.00 & 96/192 \\
Learned & $1\times$ & Minimum-$L_2$ coalition & 99.98 & 99.99 & 0/192 \\
\bottomrule
\end{tabular}
\end{table*}

\begin{table*}[t]
\centering
\caption{All contrasts in the three-allocation confirmation. Only coalition minus reference at the reference budget from learned silence is primary (marked *). The other rows are secondary. Effects and 95\% stratified graph-bootstrap intervals are in percentage points. Positive values mean more unsafe completion under the first allocation.}
\label{tab:new-allocation-main-contrasts}
\footnotesize
\begin{tabular}{@{}lclrr@{}}
\toprule
Start & Budget & Contrast & Effect & Interval \\
\midrule
Learned & $1\times$ & Coalition minus Reference$^*$ & $42.22$ & $[38.23,46.13]$ \\
Learned & $1\times$ & Uniform minus Reference & $42.58$ & $[38.73,46.41]$ \\
Learned & $1\times$ & Coalition minus Uniform & $-0.36$ & $[-1.00,0.00]$ \\
Learned & $4\times$ & Coalition minus Reference & $1.12$ & $[-1.71,4.06]$ \\
Learned & $4\times$ & Uniform minus Reference & $27.07$ & $[24.01,30.15]$ \\
Learned & $4\times$ & Coalition minus Uniform & $-25.95$ & $[-29.99,-22.16]$ \\
$p=0.95$ reset & $1\times$ & Coalition minus Reference & $66.05$ & $[62.21,69.65]$ \\
$p=0.95$ reset & $1\times$ & Uniform minus Reference & $66.76$ & $[63.08,70.14]$ \\
$p=0.95$ reset & $1\times$ & Coalition minus Uniform & $-0.71$ & $[-2.13,0.01]$ \\
$p=0.95$ reset & $4\times$ & Coalition minus Reference & $6.14$ & $[2.13,12.31]$ \\
$p=0.95$ reset & $4\times$ & Uniform minus Reference & $34.21$ & $[31.36,37.08]$ \\
$p=0.95$ reset & $4\times$ & Coalition minus Uniform & $-28.06$ & $[-33.44,-21.64]$ \\
\bottomrule
\end{tabular}
\end{table*}

\begin{table*}[t]
\centering
\caption{All tie-aware extension contrasts. Every row is secondary and uses the original graphs and paired actors. Effects and 95\% stratified graph-bootstrap intervals are in percentage points. Positive values mean more unsafe completion under the first allocation.}
\label{tab:new-allocation-tie-contrasts}
\footnotesize
\begin{tabular}{@{}lclrr@{}}
\toprule
Start & Budget & Contrast & Effect & Interval \\
\midrule
Learned & $1\times$ & Coalition minus Tie-aware coalition & $41.61$ & $[37.69,45.54]$ \\
Learned & $1\times$ & Tie-aware coalition minus Reference & $0.61$ & $[-3.42,4.42]$ \\
Learned & $1\times$ & Uniform minus Tie-aware coalition & $41.96$ & $[38.21,45.77]$ \\
Learned & $4\times$ & Coalition minus Tie-aware coalition & $0.99$ & $[-1.30,3.56]$ \\
Learned & $4\times$ & Tie-aware coalition minus Reference & $0.13$ & $[-3.81,3.80]$ \\
Learned & $4\times$ & Uniform minus Tie-aware coalition & $26.94$ & $[23.70,30.49]$ \\
$p=0.95$ reset & $1\times$ & Coalition minus Tie-aware coalition & $64.95$ & $[62.56,67.25]$ \\
$p=0.95$ reset & $1\times$ & Tie-aware coalition minus Reference & $1.10$ & $[-1.87,4.12]$ \\
$p=0.95$ reset & $1\times$ & Uniform minus Tie-aware coalition & $65.66$ & $[63.23,68.02]$ \\
$p=0.95$ reset & $4\times$ & Coalition minus Tie-aware coalition & $6.14$ & $[2.13,12.31]$ \\
$p=0.95$ reset & $4\times$ & Tie-aware coalition minus Reference & $0.00$ & $[0.00,0.00]$ \\
$p=0.95$ reset & $4\times$ & Uniform minus Tie-aware coalition & $34.20$ & $[31.36,37.08]$ \\
\bottomrule
\end{tabular}
\end{table*}

\begin{table*}[t]
\centering
\caption{All minimum-$L_2$ control contrasts, at the reference budget from learned silence. Every row is secondary. Effects and 95\% stratified graph-bootstrap intervals are in percentage points. Positive values mean more unsafe completion under the first allocation.}
\label{tab:new-allocation-l2-contrasts}
\footnotesize
\begin{tabular}{@{}lclrr@{}}
\toprule
Start & Budget & Contrast & Effect & Interval \\
\midrule
Learned & $1\times$ & Minimum-$L_2$ coalition minus Reference & $42.57$ & $[38.73,46.41]$ \\
Learned & $1\times$ & Minimum-$L_2$ coalition minus Coalition & $0.36$ & $[-0.01,0.99]$ \\
Learned & $1\times$ & Minimum-$L_2$ coalition minus Uniform & $0.00$ & $[-0.01,0.00]$ \\
Learned & $1\times$ & Minimum-$L_2$ coalition minus Tie-aware coalition & $41.96$ & $[38.20,45.76]$ \\
\bottomrule
\end{tabular}
\end{table*}

At four times the budget after the reset, reference and tie-aware audits each have 96 width-4 runs with legitimate completion near zero and 96 width-8 runs that jointly recover, so their pooled legitimate completion of 50.00\% is not a moderate loss shared across runs. Audit fines are off in legitimate contexts. Base costs, shared hidden features and the optimizer could all contribute, and the experiment does not isolate the cause. At the reference budget from learned silence, every joint recovery under reference or tie-aware audits is at width 4, and most certified runs fail the joint threshold within the horizon.

Tie-aware minus reference at the reference budget from learned silence is 0.61 points, with interval $[-3.42,4.42]$, which does not establish equivalence. A check on the inputs finds group deterrence totals that match the reference allocation within $10^{-7}$ on 22 of 24 graphs at the reference budget and on 10 at the larger budget. Equal totals give identical grouped expected objectives, although the distribution of sampled fines can still differ. Selecting among static optima changes several features of learning at once, so the comparisons do not identify certificate status as the only cause of the improvement.

\subsection{A fixed worked graph}

The retained analysis plan designates the worked case before any outcome was inspected. It is the uniform graph of width four and seed 71027011 at the reference budget from learned silence, with nine groups. All five allocations spend 4.0 per unsafe context and buy total deterrence 193.63725688. The maximal coalition multiplier is 1.5. Reference audits reach about 1.44809, while tie-aware audits keep the optimum within tolerance and meet the certificate. Mean unsafe completion is 29.67\% under reference audits and 31.28\% under tie-aware audits, each with one joint recovery in eight runs. The other three allocations leave about 99.97\% to 99.98\% unsafe completion and no joint recovery. The released case includes all audit vectors, task witnesses, group loads and all 40 endpoints. It belongs to the same 24-graph sample and adds no independent observation.

\subsection{Relation to the exact-update theorem}

The recovery theorem assumes exact gradients, separate unbounded group logits, fixed positive coordinate scalings, a constant finite step and no regularization. Neural PPO uses sampled advantages, bounded logits, entropy regularization, adaptive optimization and shared features across contexts, and reducing unsafe completion while preserving legitimate completion is a different criterion from approaching universal reporting in one unsafe context. The study measures whether an allocation chosen with the theorem helps this finite learner. It neither establishes a convergence theorem for PPO nor shows that certified allocations always do better than uncertified ones. The auditor's intervention is limited to the allocation while the team continues the fixed learning procedure. The simulator would also allow replacing the policy, which could hard-code reporting in these known unsafe contexts. The model assumes that reports stop tasks and that the auditor can verify wrongful silence even when another witness reports, and it measures no recognition, evidence collection or audit costs beyond the audits themselves.

\section{Verification and retained evidence}
\label{app:verification}

The returned main archive contains all 4,080 scheduled records, the source snapshot, the input bank, pretrained policies and execution metadata. Its source and schedule match the released study package. Record hashes, graph and prototype identities, endpoint calculations, audit budgets and continuation prefixes were checked. Evaluating the saved final policies separately gives a largest absolute difference of $5.12\times10^{-13}$ from the stored values, and all discrete certificate results agree.

The original package verifier compared full certificate dictionaries with exact floating-point equality, which can reject a valid run on another platform because of rounding. Our audit compares discrete fields exactly and real-valued fields within stated tolerances. The largest certificate discrepancy is $3.41\times10^{-13}$. The change affects only the rule for numeric comparisons, and no learned policy or result record changes.

The 288-record allocation control reproduces every original-arm joint recovery classification, and its largest cross-platform difference in mean unsafe completion is 0.0162 percentage points. The 96-record common-draw control also reproduces every original-arm classification, with a largest unsafe-fraction difference of $9.97\times10^{-11}$, and separately recomputed final values differ by at most $2.84\times10^{-14}$.

The main run took 155.54 minutes on an Apple M2 Pro using Python 3.12.12 and NumPy 2.5.0. The allocation and common-draw controls took 10.44 and 2.91 minutes on a Linux host using Python 3.12 and NumPy 2.3.5, with six workers for the common-draw control. These wall times describe only the stated runs. No GPU or model API was used for simulation training.

The supplement keeps every per-run outcome with its graph, replicate, experiment and method identifiers, which support all tables and aggregate figures without treating repeated measurements as new units. The main archive and the two control archives stay separate in the analysis, so the later controls cannot replace the prespecified estimate. The input bank itself is not included, but \texttt{reproducibility/\allowbreak study/\allowbreak tools/\allowbreak rebuild\_inputs.py} rebuilds it from the protocol. A rebuild on a different machine reproduced the schedule file byte for byte and the recorded range of general-graph $B-D$ to within $10^{-12}$. From the rebuilt bank, \texttt{analysis/\allowbreak design\_checks.py} recomputes the guard shares, single-agent certificates, elimination rounds and thresholds reported in the text. For the five budget-sweep arms that form a grouping, it also checks the certificate of every run on the general graphs in each of the run's four unsafe contexts, and it writes Table~\ref{tab:rounds} and Table~\ref{tab:splits}(a). The script \texttt{analysis/\allowbreak reproduce.py} writes every other table and figure of the grouping experiments from the included data. The four tables of Appendix~\ref{app:newalloc} are transcribed from the CSV files under \texttt{allocation\_study/analysis}, and \texttt{allocation\_study/\allowbreak verify\_evidence.py} recomputes every number in them from the endpoint rows.

A further check reconstructed all 4,080 original records and both control archives with an actor forward pass and reward calculation that do not import the study evaluator or network. Its largest saved-final-value discrepancy on the original records is $1.14\times10^{-13}$, and every joint recovery classification is unchanged. Exact rational action enumeration separately checks the expected-reward polynomial, its derivatives and the Fisher factors on finite instances. Order costs, capped purchase costs, maximum closure and the finite-time bound are checked against separate finite calculations, and directional differences check the implemented PPO and network gradients. The \texttt{revision\_audit} folder contains these scripts and their outputs, and the raw-record reconstruction needs the full result archives, which are not part of the supplement. These checks add no training observations and do not constitute a general computer-verified proof.

The checks above were recorded when the main study and its two controls were run. For the later allocation study of Appendix~\ref{app:newalloc}, the \texttt{allocation\_study} folder of the supplement includes the frozen source and inputs, every endpoint of its 3,456 records, the allocation vectors, the analysis scripts and \texttt{verify\_evidence.py}, an independently written check that recomputes all condition means, contrasts and intervals from the endpoint rows. We reran that check on a fresh extraction of the released archive. It does not recompute endpoints from saved networks, which are not included, and it does not rerun training.

\end{document}